\documentclass[a4paper,11pt,fleqn]{article}
\usepackage[ruled,vlined]{algorithm2e}
\usepackage{amsmath,amsfonts,amsthm}
\usepackage{thm-restate,thmtools}
\usepackage{algorithmic}
\usepackage{nicefrac}
\usepackage[utf8]{inputenc}
\usepackage{color}
\usepackage{dsfont}
\usepackage{xcolor,colortbl}
\usepackage{bm,bbm}
\usepackage{soul}
\usepackage{mathtools}
\usepackage{nicefrac}  
\usepackage{booktabs}
\usepackage{adjustbox}
\usepackage{hyperref}
\usepackage{url}
\usepackage{cleveref}
\usepackage[flushleft]{threeparttable}
\usepackage{multirow}
\usepackage[shortlabels]{enumitem}
\usepackage{pifont}
\usepackage{microtype}      %
\usepackage{graphicx}
\usepackage{natbib}
\usepackage{doi}
\usepackage{authblk}
\usepackage{adjustbox}

\usepackage{amsmath,amssymb,amsthm}
\usepackage{mathrsfs} 
\usepackage{euscript}
\usepackage{charter}
\usepackage{array}
\usepackage[normalem]{ulem}

\definecolor{labelkey}{rgb}{0,0.08,0.45}
\definecolor{refkey}{rgb}{0,0.6,0.0}
\definecolor{Brown}{rgb}{0.45,0.0,0.05}
\definecolor{dgreen}{rgb}{0.00,0.49,0.00}
\definecolor{dblue}{rgb}{0,0.08,0.75}

\hypersetup{
colorlinks,
linkcolor=dgreen,
citecolor=dblue
}

\newcommand{\R}{\mathbb{R}}
\newcommand{\N}{\mathbb{N}}
\newcommand{\norm}[1]{\lVert#1\rVert}
\newcommand{\EE}{\mathbb{E}}

\newcommand{\A}{\mathcal{A}}
\newcommand{\B}{\mathcal{B}}
\newcommand{\C}{\mathcal{C}}

\newcommand{\multito}{\rightrightarrows}

\newcommand{\given}{\,|\,}

\newcommand{\Ll}{C_\lambda}
\newcommand{\Lu}{C_\lambda(u)}

\newcommand{\LG}{C_{G,\lambda}}
\newcommand{\LT}{C_{T,\lambda}}

\newcommand{\LE}{C_{E,\lambda}}

\newcommand{\hz}{z_{t}}

\newcommand{\sz}{z}

\newcommand{\indic}{\mathds{1}}
\newcommand{\jvp}{(w^\prime(\lambda)^\top y)^{\widehat{}}}

\DeclareMathOperator*{\argmin}{arg\,min}

\DeclareMathOperator{\gap}{\it{e}}

\DeclareMathOperator{\Var}{Var}
\DeclareMathOperator{\fix}{fix}
\DeclareMathOperator{\imp}{imp}

\DeclareMathOperator{\prox}{Prox}
\DeclareMathOperator{\Conv}{co}
\DeclareMathOperator{\relu}{ReLU}

\newcommand{\D}[1]{D_{#1}}

\DeclarePairedDelimiter\ceil{\lceil}{\rceil}

\theoremstyle{plain}
\newtheorem{theorem}{Theorem}[section]

\newtheorem{lemma}[theorem]{Lemma}

\theoremstyle{definition}
\newtheorem{definition}[theorem]{Definition}
\newtheorem{assumption}[theorem]{Assumption}
\theoremstyle{remark}
\newtheorem{remark}[theorem]{Remark}

\title{ { \sffamily Nonsmooth Implicit Differentiation:\\ Deterministic and Stochastic Convergence Rates
} }

\author[1]{Riccardo Grazzi}
\author[1, 2]{Massimiliano Pontil}
\author[1, 3]{Saverio Salzo}

\affil[1]{\footnotesize CSML, Istituto Italiano di Tecnologia, Via Enrico Melen 83, 16152 Genova, Italy}
\affil[2]{\footnotesize Department of Computer Science, UCL, Malet Place, London WC1E 6BT, UK}
\affil[3]{\footnotesize DIAG, Sapienza University of Rome, Via Ariosto, 25, 00185 Roma, Italy}

\date{}

\begin{document}
\maketitle

\begin{abstract}
We study the problem of efficiently computing the derivative of the fixed-point of a parametric nondifferentiable contraction map. This problem has wide applications in machine learning, including hyperparameter optimization, meta-learning 
and data poisoning attacks. We analyze two popular approaches: iterative differentiation (ITD) and approximate implicit differentiation (AID). A key challenge behind the nonsmooth setting is that the chain rule does not hold anymore. 
We build upon the work by \citet{bolte2022automatic}, who prove linear convergence of nonsmooth ITD under a piecewise Lipschitz smooth assumption. 
In the deterministic case, we provide a linear rate for AID and an improved linear rate for ITD which closely match the ones for the smooth setting.  
We also introduce NSID, a new stochastic method to compute the implicit derivative when the contraction map is defined as the composition of an outer map and an inner map which is accessible only through a stochastic 
unbiased estimator. 
We establish rates for the convergence of NSID, encompassing  the best available rates in the smooth setting.
We present illustrative experiments confirming our analysis.
\end{abstract}

 \vspace{1ex}

\section{Introduction}
\label{se:intro}

In this paper, we study the problem of efficiently approximating a generalized derivative (or Jacobian) of the solution map 
of the parametric fixed point equation
\begin{equation}\label{eq:fixedpoint}
    w(\lambda) = \Phi(w(\lambda), \lambda)\quad(\lambda \in \R^m),
\end{equation}
when $\Phi$ is not differentiable, but only %
piecewise differentiable.
We address both the case that $\Phi$ can be explicitly evaluated, and the case that 
$\Phi$ has the composite form
\begin{equation}\label{eq:fixedpointstoch}
\begin{aligned}
     \Phi(w,\lambda) &= G(T(w,\lambda),\lambda)\\
     T(w,\lambda)& = \EE[\hat T_\xi (w(\lambda),\lambda)],
\end{aligned}
\end{equation}
where the external map $G$ can be evaluated, but the inner map $T$ 
is accessible only via a stochastic estimator $\hat T_\xi$, with $\xi$ a random variable.

A main motivation for computing the \textit{implicit} derivative of \eqref{eq:fixedpoint} is provided by 
bilevel 
optimization, which aims to minimize an upper level objective function of $w(\lambda)$.  
Important examples are given by hyperparameter optimization and meta-learning \citep{franceschi2018bilevel,lee2019meta}, where \eqref{eq:fixedpoint} expresses the optimality conditions of a lower-level minimization problem. %
Further examples include learning a surrogate model for data poisoning attacks \citep{xiao2015feature,munoz2017towards}, 
deep equilibrium models \citep{bai2019deep} or OptNet \citep{amos2017optnet}. All these problems may present nonsmooth mappings $\Phi$. For instance, consider hyperparameter optimization or data poisoning attacks for SVMs, or meta-learning for image classification, where $\Phi$ is evaluated through the forward pass of a neural net with RELU %
activations %
\citep{bertinetto2019meta,lee2019meta,rajeswaran2019meta}. In addition, when such settings are applied to large datasets, evaluating the map $\Phi$ would be too costly, but we can usually apply stochastic methods through the composite stochastic structure in \eqref{eq:fixedpointstoch}, where only $T$ involves a computation on the full training set (e.g., a gradient descent step).

Nowadays, automatic differentiation techniques \citep{griewank2008evaluating} popular for deep learning, can also be used to efficiently,
i.e. with a cost of the same order of that of approximating $w(\lambda)$,
approximate Jacobian-vector (or vector-Jacobian) products of $w(\lambda)$ by relying only on an implementation of an iterative solver for problem~\eqref{eq:fixedpoint}.
There are two main approaches to achieve this: ITerative Differentiation (ITD) (e.g.,\@ \citet{maclaurin2015gradient,franceschi2017forward}), which differentiates through the steps of the solver for~\eqref{eq:fixedpoint}, and Approximate Implicit Differentiation (AID) (e.g.,\@ \citet{pedregosa2016hyperparameter,lorraine2019optimizing}), which relies on approximately solving the linear system emerging from the implicit expression for the Jacobian-vector product. Despite the analysis of such methods has been usually done in the case that $\Phi$ is smooth, there are now several open source implementations relying on popular deep learning frameworks (e.g.,\@ \citet{grazzi2020iteration,blondel2022efficient,liu2021boml}), which practitioners can use even when $\Phi$ is not differentiable. However, when $\Phi$ is not differentiable despite existing algorithmic proposals \citep{ochs2015bilevel,frecon2018bilevel}, establishing theoretical convergence guarantees is challenging, since even if the solution map $w$ is almost everywhere differentiable and the  Clarke subgradient is well defined, the chain rule of differentiation, exploited by AID and ITD approaches, does not hold. 

Recently \citet{bolte2021conservative} introduced the notion of conservative derivatives as an effective tool to rigorously address automatic differentiation of neural networks with nondifferentiable activations (e.g., ReLU).
Moreover, if $\Phi(\cdot,\lambda)$ is a contraction and under the general assumption that $\Phi$ is
piecewise Lipschitz smooth with finite pieces, \citet{bolte2022automatic} provide an asymptotic linear convergence rate for deterministic ITD.\footnote{Therein,  referred to as piggyback automatic differentiation.} 
However, such rate is worse than that of the smooth case and we are not aware of any result of such type for the AID method and for the  stochastic setting of problem~\eqref{eq:fixedpointstoch}, even when $G(v,\lambda) = v$. In particular the compositional structure \eqref{eq:fixedpointstoch} allows us to cover e.g., proximal stochastic gradient methods, which are a common and practical example of nonsmooth optimization algorithms, but it adds additional challenges since we do not have access to an unbiased estimator of $\Phi$ as for the smooth stochastic case studied in \citep{grazzi2021convergence,grazzi2023bilevel}.

\paragraph{Contributions}
We present theoretical guarantees on AID and ITD for the approximation of the conservative derivative of the fixed point solution of \eqref{eq:fixedpoint}, building upon the framework of \citet{bolte2022automatic}. Specifically:
\begin{itemize}
\item We prove non-asymptotic linear convergence rates for deterministic ITD and AID which, from one hand extend the rates for the case where $\Phi$ is Lipschitz smooth given in \citep{grazzi2020iteration}, which are fully recovered as a special case, and on the other end, improve the result in \citep{bolte2022automatic} for nonsmooth ITD. The given bounds indicate that AID converges faster than ITD, which we verify empirically. We also identify cases in which this difference in performance in favor of AID might be large due to nondifferentiable regions. 
\item We propose the first nonsmooth stochastic AID approach with proven convergence rates, which we name \emph{nonsmooth stochastic implicit differentiation} (NSID). Notably, we prove that NSID can converge to a true conservative Jacobian-vector product with rate $O(1/k)$, where $k$ is the number of samples, provided that the fixed-point problem is solved with rate $O(1/k)$. 
\item Finally, we provide experiments on two bilevel optimization problems, i.e. hyperparameter optimization and adversarial poisoning attacks, confirming our theoretical findings.
\end{itemize}

\paragraph{Related Work}
When $\Phi$ is differentiable and under some regularity assumptions, approximation guarantees have been established for AID and ITD approaches in the deterministic setting \citep{pedregosa2016hyperparameter,grazzi2020iteration}, and for AID in the special case of the stochastic setting \eqref{eq:fixedpointstoch} where $G(v,\lambda) = v$ 
\citep{grazzi2021convergence,grazzi2023bilevel}. Furthermore, several works established convergence  rates and, in the stochastic setting, sample complexity results for bilevel optimization algorithms relying on AID and ITD approaches, see e.g., \citep{ghadimi2018approximation, ji2021bilevel,arbel2021amortized,chen2021closing}.\\
Aside from \citep{bolte2022automatic}, in the nonsmooth case, \citet{bertrand2020implicit,bertrand2022implicit} present deterministic and sparsity-aware nonsmooth ITD and AID procedures together with asymptotic linear convergence guarantees when $w(\lambda)$ is the solution of a composite minimization problem where one component has a sum structure. Contrary to this work and to \citep{bolte2022automatic}, their results rely on some differentiability assumptions on the algorithms, which are verified after a finite number of iterations. For bilevel optimization, some recent works have provided stochastic algorithms with convergence rates for the special case where the lower-level problem has linear \citep{khanduri2023linearly} or equality \citep{xiao2023alternating} constraints. 

\section{Preliminaries}
\label{sec:prel}

\paragraph{Notation}
If $U$ and $V$ are two nonempty sets, we denote by $F\colon U \multito V$ a \emph{set-valued mapping} which associates to an element of $U$ a subset of $V$. A \emph{selection} of $F$ is a single-valued function $f\colon U\to V$ such that, for every $x\in U$, $f(x) \in F(x)$. We denote with $\norm{\cdot}$ the Euclidean and operator norm when applied to vectors and matrices, respectively. Set inclusion is denoted by $\subset$.
We define Minkowski operations on sets of matrices as follows:
if $\A, \B \subset \R^{n \times p}$ and $\C \subset \R^{p \times d}$ then
\begin{eqnarray*}
  \A + \B &:=& \{A + B \,\vert\, A \in \A, B \in \B \},\\  
  \A\, \C &:=& \{AC \,\vert\, A \in \A, C \in \C \}\\
  \A^* &:=& \{A^* \,\vert\, A \in \A \}\ \quad \text{with}\ * \in \{\top, -1\}.
\end{eqnarray*}
We let $\Conv(\A)$ be the convex envelope of $\A$, and define $\norm{\A}_{\sup} = \sup\{\norm{A} \,\vert\, A \in \A\}$. 
It will be convenient to define for every $\A \subset \R^{n \times (p_1+p_2)}$ the map acting between sets of matrices, which we still denote by $\A$, such that for every $\mathcal{X} \subset \R^{p_1\times p_2}$
\begin{equation}\label{eq:matrixaffine}
    \A(\mathcal{X}) {: =} \A \, \begin{bmatrix} \mathcal{X} \\ I_{p_2}\end{bmatrix} {:=} \big\{ A_1 X + A_2 \,\big\vert\, [A_1, A_2] \in \A, X \in \mathcal{X}\big\},
\end{equation}
where $I_{p_2}$ is the identity matrix of dimensions $p_2\times p_2$.

For any integer $r \geq 1$ we set $[r] = \{1,\dots, r\}$. If $F\colon \R^{p_1+p_2}\to \R^d$ is differentiable, we denote by $F^\prime(x)\in \R^{d\times(p_1+p_2)}$ the derivative of $F$ (its Jacobian) at $x$ and by $\partial_1 F(x)\in \R^{d\times p_1}$ and $\partial_2 F(x)\in \R^{d\times p_2}$ the partial derivatives of $F$ with respect to the first and second block of variables respectively. 
Let $F: U \subset \R^d \to \R^p$, we say that $F$ is smooth (nonsmooth) if it is differentiable (not differentiable) and  \emph{Lipschitz smooth with constant $L$} or \emph{$L$-smooth} if it is smooth and $ \exists L> 0$  such that $\forall x,y \in  U$ $\norm{F'(x) - F'(y)} \leq L\norm{x-y}$.  
For a random vector $\xi \in \R^d$, we denote with $\EE[\xi]$ its expectation and with $\Var[\xi] = \EE\norm{\xi - \EE[\xi]}^2$ its variance.
In our assumptions we will consider the class of the so called \textit{definable} functions, which includes the large majority of functions used for machine learning applications (see \Cref{se:definable}).

\subsection{Conservative Derivatives}
We provide some definitions related to path differentiability and sets of matrices and vectors. They are mostly borrowed, possibly with slight modifications, from \citep{bolte2021conservative}, where additional details can be found.

\begin{definition}[Conservative Derivatives]
Let $U\subset \R^p$ be an open set and $F\colon U\subset \R^p \to \R^d$ be a locally Lipschitz continuous mapping.
We say that a set-valued mapping $\D{F}\colon U \multito \R^{d \times p}$ is a \emph{conservative derivative} of $F$, 
if $\D{F}$ has closed graph, nonempty compact values, and for every absolutely continuous curve $\gamma\colon[0,1] \to U \subset \R^p$ we have that, for almost every $t \in [0,1]$
\begin{equation}
\label{eq:CD}
    \frac{d}{dt} F(\gamma(t)) = V \gamma^\prime(t),\quad\forall\,V\in \D{F}(\gamma(t)).
\end{equation}
The function $F$ is called \emph{path differentiable}
if it admits
a conservative derivative.
\end{definition}

Conservative derivatives are extensively analyzed in \citep{bolte2021conservative}. Some key properties 
are that: (1) they are almost everywhere single-valued and equal to classical derivatives; (2) for path differentiable functions, the Clarke subgradient is the minimal  conservative derivative up to a convex envelope; (3) chain rule holds for conservative derivatives; (4) locally Lipschitz definable mappings admit conservative derivatives. We also point out that -- as it is usual for generalized derivatives -- conservative derivatives are unique only up to a set of Lebesgue measure zero. 
This accounts for the fact that there are multiple ways to express a path differentiable function as a composition of others but applying the chain rule produces conservative derivatives that can differ but are always valid.

Similarly to \citep{bolte2022automatic}, to address the fact that conservative derivatives are set-valued mappings, we will use the following quantity to measure the error in the conservative derivative approximation.
\begin{definition}[Excess] Let $\A$ and $\B$ be two bounded subsets of matrices or vectors. 
The \emph{excess}\footnote{$\gap$ is referred in \citep{bolte2022automatic} as gap, while the standard name is excess \citep[Section 1.5]{beer1993topologies}.} \emph{of $\A$ over $\B$} is
\begin{equation*} 
    \gap(\A, \B) := \sup_{A\in \A} \inf_{B \in B} \norm{A - B}.
\end{equation*}    
\end{definition}
Note that $\gap(\A,\B) = 0 \implies \A\subset \B$ and $\norm{\A}_{\sup} = \gap(\A,\{0\})$. The excess 
satisfies several properties similar to the ones of a distance, even though it is not symmetric (see \Cref{lm:gapprop}). 
Similarly to \citep{Scholtes2012} we give the following concept of piecewise continuity and smoothness (which is slightly more general than that given in \citep{bolte2021conservative}).
\begin{definition}
Let $F_1, \dots, F_r\colon U\subset \R^p\to \R^d$
be continuous mappings defined on
 a nonempty open set $U$. A \emph{continuous selection of $F_1, \dots, F_r$} is a continuous mapping $F\colon U\to \R^d$ such that
for every $x \in U\colon F(x) \in \{F_1(x), \dots, F_r(x)\}$. In such case the \emph{active index set mapping}
is the set-valued mapping $I_F\colon U \multito [r]$, with 
$I_F(x)=\{i \in [r] \,\vert\, F_i(x) = F(x)\}$. Moreover, if the $F_i$'s are differentiable we set $\D{F}^s\colon U\subset \R^p \multito \R^{d\times p}$ such that
\begin{equation}
\label{eq:sCD}
\D{F}^s(x) = \Conv(\{F^\prime_i(x) \,\vert\, i \in I_F(x)\}),
\end{equation}
where $F^\prime_i(x)$ is the classical derivative (Jacobian) of $F_i$ at $x$.
\end{definition}

\begin{theorem}\label{thm:20240129a}
Let  $F\colon U\subset\R^p \to \R^d$ be a continuous selection of 
definable and continuously differentiable mappings $F_1, \dots, F_r\colon U \to \R^d$.
Then $F$ is definable if and only if $I_F\colon\R^p\multito [r]$ is definable, and in such case
 $\D{F}^s$ is a conservative derivative of $F$.
\end{theorem}

We can also define partial conservative derivatives. 
If $p=p_1+p_2$ and $F\colon U\subset \R^{p_1+p_2}\to \R^d$,
we have $\D{F}\colon U \multito \R^{d\times(p_1+p_2)}$
and we set $\D{F,1}\colon U\multito \R^{d\times p_1}$ and
$\D{F,2}\colon U\multito \R^{d\times p_2}$ such that for $j \in \{1,2\}$
\begin{equation*}
\D{F,j}(x) = \big\{A_j \,\vert\, [A_1,A_2] \in \D{F}(x)\big\}.
\end{equation*}
Finally, we denote by $F'(x)$ an arbitrary element of $\D{F}(x)$ and by $\partial_1 F(x) \in \R^{d \times p_1}$ and $\partial_2 F(x) \in \R^{d \times p_2}$ the first and second block component of $F'(x)$ respectively, 
which yield the classical (partial)  derivatives if $F$ is differentiable.
By building on \citep[Lemma 3]{bolte2022automatic}, we prove the following result (the proof is in Appendix~\ref{app:aux}).
\begin{lemma}%
\label{lm:glolip} 
Let $F\colon U\subset\R^p \to \R^d$ be
 a continuous definable selection of the definable Lipschitz smooth mappings
$F_1,\dots, F_r\colon U\to \R^d$. Let $L_i$
be the Lipschitz constant of $F'_i$ and set $L = \max_{1\leq i \leq r} L_i$.
Then for every $x \in U$, there exist $R_x > 0$ such that for every $x' \in U$ 
\begin{equation}\label{eq:glolip}
\gap(\D{F}^s(x'), \D{F}^s(x)) \leq L_x(x')\norm{x-x'},
\end{equation}
where 
\begin{equation*}
    L_x(x') := \begin{cases}
        L & \text{if }  \norm{x-x'} \leq R_x \\
        L + M_x/R_x &\text{otherwise}
    \end{cases}
    \quad\text{and}\quad M_x : = \max_{i\in [m]}\min_{j \in I_F(x)} \norm{F^\prime_i(x)-F^\prime_j(x)}.
\end{equation*}
\end{lemma}

Note that in the smooth case ($r=1$), \eqref{eq:glolip} corresponds to global $L$-smoothness (since $M_x=0$), while in general it is weaker. In particular, the quantity $L + \frac{M_x}{R_x}$ is well defined even when $F$ is not differentiable at $x$, but blows up when $x$ approaches a point of non-differentiability, e.g.,~for $\relu(x)=\max(0, x)$, $\lim_{x \to 0^+} M_x/R_x = \infty$, since if $x \neq 0$ $M_x = 1$ and $R_x = \lvert x \rvert$, while for $x=0$, $M_x/R_x = 0$ since $M_x=0$ and $R_x > 0$ can be chosen arbitrarily. 

\section{Differentiating a Parametric Fixed Point}
\label{sec:3}

\paragraph{Instances of Parametric Fixed Point Equations}
A general class of problems that can be recast in the form \eqref{eq:fixedpoint} is that of the parametric monotone inclusion problem
\begin{equation}
\label{eq:20240201b}
    0 \in A_\lambda (w) + B_\lambda (w),
\end{equation}
where $A_\lambda\colon \R^d\multito \R^d$ and $B_\lambda\colon \R^d \to \R^d$ are multi-valued and single-valued maximal monotone operators respectively. These types of problems are at the core of convex analysis and can cover a number of optimizations problems including minimization problems as well as variational inequalities and saddle points problems. It is a standard fact (see \citep{Bauschke2017}) that \eqref{eq:20240201b} can be rewritten as the equation
\begin{equation*}
R_{\gamma A_{\lambda}}(w - \gamma B_\lambda (w)) = w \quad (\gamma>0),
\end{equation*}
where $R_{\gamma A_{\lambda}}$ is the resolvent of 
the operator $\gamma A_{\lambda}$. This gives a fixed-point equation of a composite form, and comparing with \eqref{eq:fixedpointstoch}, it is clear that we can also address situations in which $B_\lambda = \EE[\hat{B}_{\lambda}(\cdot, \xi)]$.
\citet{bolte2024differentiating} investigates conservative derivatives of the solution map of such monotone inclusion problems in nonsmooth settings.

A special case of \eqref{eq:20240201b} 
is the minimization problem 
\begin{equation}\label{eq:composite}
    \min_{w} \EE [\hat f_\lambda(w, \xi)] +  g_\lambda(w),
\end{equation}
where $f_\lambda = \EE \hat f_\lambda(\cdot,\xi)$ is  convex $L$-smooth, while $g_\lambda$ is convex lower semicontinuous extended-real valued. This can be cast into \eqref{eq:fixedpointstoch} by setting $\eta \in \left[0,2/L\right[$,  $\hat T_\xi(w,\lambda) = w - \eta \hat \nabla f_\lambda(w,\xi)$ and $G(w,\lambda) = \prox_{\eta g_\lambda}(w)$ with $\prox_{h}(x) =\argmin_{y}(h(x) + (1/2)\norm{x-y}^2)$ being the proximity operator of $h$. Several machine learning problems can be written in form \eqref{eq:composite} where $g_\lambda$ is nonsmooth, e.g., \@ LASSO, elastic net, (dual) SVM.

\paragraph{Main assumptions}
Referring to problem~\eqref{eq:fixedpoint}, when $\Phi$ is differentiable and $\norm{\partial_1\Phi(w(\lambda),\lambda)}\leq q<1$,
by differentiating \eqref{eq:fixedpoint} we have
\begin{equation}
\label{eq:20240201a}
\begin{aligned}
w^\prime(\lambda) & = \partial_1\Phi(w(\lambda),\lambda)w^\prime(\lambda) + \partial_2\Phi(w(\lambda),\lambda) \\
w^\prime(\lambda) &= (I- \partial_1\Phi(w(\lambda),\lambda))^{-1}\partial_2\Phi(w(\lambda),\lambda).
\end{aligned}
\end{equation}
The first relation above shows that $w^\prime(\lambda) \in \R^{d\times p}$ is a fixed point of the map
$X\mapsto \partial_1\Phi(w(\lambda),\lambda)X + \partial_2\Phi(w(\lambda),\lambda)$.
Here, dealing with the nonsmooth case, we will mimic the above formulas.
The crucial assumption of our analysis is the following.
\begin{assumption}\label{ass:lipsel} Let $O_\Lambda \subset \R^m$ be an open set and $\Lambda \subset O_\Lambda$ be a nonempty closed and convex set.
\begin{enumerate}[label={\rm (\roman*)}]
\item\label{ass:lipsel_i} $\Phi\colon \R^d\times O_\Lambda\to \R^d$ is definable and a continuous selection of the $L$-Lipschitz smooth definable mappings 
$\Phi_1,\dots, \Phi_r$ and we set
$\D{\Phi}\colon\R^d\times O_\Lambda\multito \R^{d\times(d+m)}$, %
\begin{equation}
\label{eq:product}
\D{\Phi}(u,\lambda) {=} \D{\Phi}^s(u,\lambda)
{=}\Conv(\{\Phi^\prime_i(u,\lambda) \,\vert\, i \in I_\Phi(u,\lambda)\}).
\end{equation}
\item\label{ass:lipsel_iii} For all $(u,\lambda) \in \R^p\times O_\Lambda$,
$\norm{\D{\Phi,1}(u,\lambda)}_{\sup} \leq q<1$.
\end{enumerate}
\end{assumption}

Theorem~\ref{thm:20240129a} ensures that $\D{\Phi}$, 
as defined in \eqref{eq:product}, is a conservative derivative of $\Phi$.
Moreover, recalling \eqref{eq:CD}, it is easy to see that  \Cref{ass:lipsel}\ref{ass:lipsel_iii}  ensures that $\Phi(\cdot, \lambda)$ is a $q$-contraction
and hence that there exists a unique fixed point of $\Phi(\cdot,\lambda)$ that we will denote by  $w(\lambda)$.
Finally, if $A \in \D{\Phi, 1} (u,\lambda)$, we have $\norm{A}<1$ and hence $I-A$ is invertible.
Thus, mimicking what happens  for the smooth case in \eqref{eq:20240201a} one defines
\begin{gather}
  \label{eq:Jimp}
\D{w}^{\imp} (\lambda) {=} \big\{ (I-A_1)^{-1}A_2 \,\vert\, [A_1, A_2] \in \D{\Phi}(w(\lambda),\lambda) \big\}\\
\D{w}^{\fix} \colon \lambda \multito \mathrm{fix}[\D{\Phi}(w(\lambda),\lambda)], \label{eq:Jfix}
\end{gather}
where 
$\mathrm{fix}[\D{\Phi}(u,\lambda)]$ is the unique fixed ``point'' 
of the map  $\mathcal{X} \mapsto \A(\mathcal{X})$, where $\A = \D{\Phi}(u,\lambda)$ (see equation \eqref{eq:matrixaffine}),
which acts between compact sets of $d\times m$ matrices.
In \citep{bolte2021nonsmooth} it is proved that if $\Phi$ is path differentiable and Assumption~\ref{ass:lipsel}\ref{ass:lipsel_iii} holds, 
the set-valued mappings $\D{w}^{\imp}$
and $\D{w}^{\fix}$ are both conservative derivatives of $w(\lambda)$ and 
 $\D{w}^{\text{imp}}(\lambda) \subset \D{w}^{\fix}(\lambda)$.

\Cref{ass:lipsel} yields the following lemma through a direct application of \Cref{lm:glolip}.
\begin{lemma}\label{lm:glolipphi}
Under \Cref{ass:lipsel}\ref{ass:lipsel_i}, for every $\lambda \in \Lambda$, there exist $R_\lambda > 0$ such that for every $u \in \R^d$
\begin{equation*}
\gap(\D{\Phi}(u,\lambda), \D{\Phi}(w(\lambda),\lambda)) \leq \Lu \norm{u -w(\lambda)},
\end{equation*}
where
\begin{equation}\label{eq:lipJphi}
    \Lu := \begin{cases}
        L & \text{if } \norm{u -w(\lambda)} \leq R_\lambda \\
        L +M_\lambda/R_\lambda & \text{otherwise}
    \end{cases}
\end{equation}
and $M_\lambda\! := \max\limits_{i\in[r]}\min\limits_{j \in I_\Phi(w(\lambda),\lambda)} 
\norm{\Phi^\prime_i(w(\lambda),\lambda)-\Phi^\prime_j(w(\lambda),\lambda)}$.
\end{lemma}
Lemma~\ref{lm:glolipphi} can be used as a substitute for the Lipschitz smoothness of $\Phi$ with respect to the first variable, indeed note that in our analysis $\lambda$ (and hence $w(\lambda)$) is fixed. 

\begin{remark}\label{rm:sumphi}
Our theoretical analysis requires only that $\Phi$ is definable piecewise smooth and that the inequality in \Cref{lm:glolipphi} holds for some conservative derivatives of $\Phi$,
even if it is not computed according to \eqref{eq:product}. One such situation occurs for instance when $\Phi$ has the structure of a finite sum, that is, $\Phi = \sum_{i=1}^n \Phi^{(i)}$, where each $\Phi^{(i)}$ satisfies Assumption~\ref{ass:lipsel}\ref{ass:lipsel_i}  with corresponding conservative derivative $\D{\Phi^{(i)}}^s$. 
Then, it is clear that $\Phi$ is still definable and piecewise Lipschitz smooth. Moreover, using the properties of conservative derivatives (see Corollary~4 in \citep{bolte2020mathematical}), $\D{\Phi} = \sum_{i=1}^n \D{\Phi^{(i)}}^s$ is a conservative derivative of $\Phi$. Thus, using the property of the excess (see Lemma~\ref{lm:gapprop}\ref{eq:gapsum}) it directly follows that the inequality in \Cref{lm:glolipphi}, and hence our theory, still holds for such $\Phi$.
\end{remark}

\section{Deterministic Iterative and Approximate Implicit Differentiation}
\label{se:detitdaid}
We now formalize two deterministic methods 
for approximating the
conservative derivative of the solution map $w$.

\paragraph{Iterative Differentiation (ITD)} 
This method approximates $\D{w}^{\fix}(\lambda)$ through the following iterative procedure, starting from $w_0(\lambda) \in \R^d$, $\D{w_0}(\lambda) = \{0\}$, 
\begin{equation}\label{eq:ITD} 
\begin{array}{l}
\text{for}\;t=1,2\ldots\\
\left\lfloor
\begin{array}{l}
w_{t}(\lambda) = \Phi(w_{t-1}(\lambda), \lambda)\\[1ex]
\D{w_{t}}(\lambda) = 
\D{\Phi}(w_{t-1}(\lambda),\lambda)\, \begin{bmatrix} \D{w_{t-1}}(\lambda) \\ I_m\end{bmatrix},
\end{array}
\right.
\end{array}
\end{equation}
where we used the definition in \eqref{eq:matrixaffine}. Note that the iteration for $\D{w_{t}}(\lambda)$ is based on the chain rule and results in a conservative derivative of $w_t(\lambda)$.
This is the same set-valued iteration studied in \citep{bolte2022automatic}. We  note that if $\Phi(\cdot,\lambda)$ is a $q$-contraction, it holds $\norm{w_t(\lambda)-w(\lambda)} = O(q^t)$.

\paragraph{Approximate Implicit Differentiation with Fixed Point (AID-FP)} 
An alternative method for approximating the implicit conservative derivative is the following. Assume that
$w_t(\lambda)$ is generated by any algorithm converging to $w(\lambda)$ (for instance the one in \eqref{eq:ITD}), then, starting from $\D{w_t}^{0}(\lambda) = \{0\}$, define
\begin{equation}\label{eq:aidfp}
\begin{array}{l}
\text{for}\;k=1,2\ldots\\
\left\lfloor
\begin{array}{l}
\D{w_{t}}^{k}(\lambda) = \D{\Phi}(w_{t}(\lambda),\lambda)\, \begin{bmatrix} \D{w_{t}}^{k-1}(\lambda) \\ I_m\end{bmatrix}.
\end{array}
\right.
\end{array}
\end{equation}

\paragraph{Efficient Implementation} 
In practice we do not compute the full set-valued iterations in \eqref{eq:ITD} and \eqref{eq:aidfp}, but rather we select just one element at each iteration.
Moreover, if we let $x \in \R^m$ and $ y\in \R^d$, the ITD method can exploit automatic differentiation to efficiently compute an element of the conservative Jacobian-vector products $\D{w_{t}}(\lambda)^\top y$ (in reverse mode) and $\D{w_{t}}(\lambda) x$ (in forward mode). Similarly AID can efficiently compute an element in $\D{w_t}^{k}(\lambda)^\top y$.
Thanks to Automatic Differentiation, if $k=t$ the standard implementation of both AID-FP and ITD has a cost in time of the same order of that of computing $w_t(\lambda)$. However, while AID-FP only uses $w_t(\lambda)$, ITD has a larger $\Theta(t)$ memory cost, since it needs to store the entire optimization trajectory $(w_i(\lambda))_{0 \leq i \leq t}$.

\paragraph{Convergence Guarantees} In the Lipschitz smooth case \citet{grazzi2020iteration} proved non-asymptotic linear convergence rates for both methods, revealing that AID-FP is slightly faster than ITD.  We now extend this analysis to nonsmooth ITD and AID-FP, focusing on the convergence of the set-valued iterations in \eqref{eq:ITD} and \eqref{eq:aidfp}. Thanks to \Cref{lm:glolipphi} and the properties of the excess, the proof (in Appendix~\ref{app:iterAID}) can proceed similarly %
to the one for the smooth case.

\begin{theorem}[nonsmooth ITD and AID-FP Rates]\label{th:itdaidrates} 
Let \Cref{ass:lipsel} hold. For every $\lambda \in \Lambda$, let  $R_\lambda$ and $M_\lambda$ be the quantities defined in \Cref{lm:glolipphi} and $B_\lambda := \norm{\D{\Phi,2}(w(\lambda),\lambda)}_{\sup}$. For every $t,k \in \N$, let $\Delta_t = \norm{w_t(\lambda) - w(\lambda)}$, $\delta_\lambda(t) : = \indic\{\Delta_t > R_\lambda \}$ and $\bar{\delta}_\lambda(t) = t^{-1}\sum_{i=0}^{t-1}\delta_\lambda(i) $.  Then the following hold. 
\begin{enumerate}[label={\rm (\roman*)}]
\item
 The ITD iteration in \eqref{eq:ITD} satisfies
\begin{equation}\label{eq:itdbound}
\begin{aligned}
    \gap(\D{w_t}&(\lambda), \D{w}^{\fix} (\lambda)) \leq  \frac{B_\lambda}{1-q} \, q^t  
 +  \frac{B_\lambda+1}{1-q} \Big(L + \frac{M_{\lambda}}{R_{\lambda}} \bar{\delta}_\lambda(t) \Big) \Delta_0 \, t \,q^{t-1}.
\end{aligned}
\end{equation}
\item The AID-FP iteration in \eqref{eq:aidfp} satisfies
\begin{equation}\label{eq:aidfpbound}
\begin{aligned}
    \gap(\D{w_t}^{k}&(\lambda), \D{w}^{\fix} (\lambda)) \leq \frac{B_\lambda}{1-q} \,  q^k 
+ \frac{B_\lambda+1}{1-q} \Big(L + \frac{M_{\lambda}}{R_{\lambda}} \delta_\lambda(t) \Big)\frac{1- q^k}{1-q} \Delta_t.
\end{aligned}
\end{equation}
\end{enumerate}
Moreover, if $w_t(\lambda) = \Phi(w_{t-1}(\lambda), \lambda)$, then $\Delta_t \leq q\Delta_{t-1}\leq q^t \Delta_0$ and there exists $\tau_\lambda \in \N$ such that $\delta_\lambda(t) = \indic\{ t < \tau_\lambda \}$ and thus $\delta_\lambda(t) \leq \bar{\delta}_\lambda(t) \leq 1$.
\end{theorem}

 To compare the two rates in \Cref{th:itdaidrates}, let $t=k$ and  $w_t(\lambda) = \Phi(w_{t-1}(\lambda), \lambda)$, so that both AID-FP and ITD have time complexity of the order of computing $w_t(\lambda)$. In that situation, since $1-q^k =1-q^t < q^{-1}(1-q)t$ and $\delta_\lambda(t) \leq \bar{\delta}_\lambda(t)$, the upper bound of  AID-FP is always lower than that of ITD. Moreover, if we let $\kappa = (1-q)^{-1}$ to play a similar role to the condition number, we observe that both methods converge linearly: AID-FP as $O(\kappa^2 e^{-t/\kappa})$, while ITD slightly slower as $O(\kappa te^{-t/\kappa})$.
 When $t \geq \tau_\lambda$, $\delta_\lambda(t) = 0$ while $\bar{\delta}_\lambda(t) = \tau_\lambda/t$,
which might cause a wide difference between the two bounds if $M_\lambda / R_\lambda$ is large, and such 
ratio can get arbitrarily large the closer $(w(\lambda), \lambda)$ is to regions where $\Phi$ is not differentiable. 
Finally, if we replace \Cref{lm:glolipphi} with the $L$-smoothness of $\Phi$, 
we essentially recover the same bounds reported by \citet{grazzi2020iteration}, where the terms $\delta_\lambda, \bar{\delta}_\lambda$ do not appear. 

The work by \citep{bolte2022automatic} also reports a rate for nonsmooth ITD of $O((\sqrt{q} +\epsilon)^t)$ for arbitrary $\epsilon >0$. However, this rate does not match the best available rate for smooth ITD \citep{grazzi2020iteration}. \Cref{th:itdaidrates} (in \eqref{eq:itdbound}) fills this gap since it achieves\footnote{For any $\epsilon \in [0, 1-q]$, $\exists C > 0$ such that $tq^{t-1} \leq C (q+\epsilon)^t$.} an improved rate of $O((q + \epsilon)^t)$. Moreover, our rate is more explicit, since it does not involve any arbitrary $\epsilon$.

We conclude the section by noting that \Cref{th:itdaidrates}  ensures that the sequence constructed by  selecting one element at each iteration in \eqref{eq:ITD} and \eqref{eq:aidfp}, is guaranteed to converge, up to a subsequence, to the set $\D{w}^{\fix} (\lambda)$.

\section{Nonsmooth Stochastic Implicit Differentiation}
\label{se:sid}
In this section we study the stochastic fixed point formulation in \eqref{eq:fixedpointstoch} and present an algorithm that, given a random vector $y \in \R^d$ and an approximate solution $w_t(\lambda)$, efficiently approximates an element of $\D{w}^{\imp}(\lambda)^\top y$ accessing only $\hat T_\xi$, $G$ and fixed selections of their conservative derivatives.
Similarly to deterministic AID, here we assume that $w_t(\lambda)$ is generated by a stochastic algorithm which converges in mean square to $w(\lambda)$. Several algorithms can ensure such convergence guarantees for the composite minimization problems in~\eqref{eq:composite} (e.g, \citet{rosasco2020convergence} provide a proximal stochastic gradient algorithm with rate $O(1/t)$) and composite monotone inclusions \citep{rosasco2014stochastic}. \\
We recall that for a path differentiable function $F \colon U \subset \R^{p_1 + p_2}\to \R^d$, we denote by $F'$ an arbitrary selection of $\D{F}$ and by $\partial_1 F(x) \in \R^{d \times p_1}$ and $\partial_2 F(x) \in \R^{d \times p_2}$ the first and second block component of $F'(x)$ respectively, so that we can write $F'(x) = [\partial_1 F(x), \partial_2 F(x)]$. 

We consider the following assumptions
\begin{assumption}\label{ass:contlip}\  
\vspace{-.2truecm}
\begin{enumerate}[label={\rm (\roman*)}]
    \item\label{ass:lipgt} $T$ and $G$ satisfy Assumption~\ref{ass:lipsel}\ref{ass:lipsel_i} individually, with constant $L_T$ and $L_G$ respectively. 
    Let $T^\prime$ and $G^\prime$ be selections of the conservative derivatives $\D{T}$ and $\D{G}$ respectively. 
    Also, $\Phi(u,\lambda) = G(T(u,\lambda),\lambda)$.
    \vspace{-.1truecm}\item\label{ass:contgt} For every $(u,\lambda)\in \R^d\times\Lambda$, $\norm{\D{T,1}(u,\lambda)}_{\sup}\leq 1$ and $\norm{\D{G,1}(u,\lambda)}_{\sup} \leq 1$ and either $T$ or $G$ satisfies Assumption~\ref{ass:lipsel}\ref{ass:lipsel_iii}.
    \item $y\in \R^d$ is a random vector.
\end{enumerate}
\end{assumption}
\begin{assumption}\label{ass:variance}\ 
The random variable $\xi$ takes  values in $\Xi$ and for every $x \in \Xi$
\begin{enumerate}[label={\rm (\roman*)}]
\item $\hat{T}_x\colon \R^d\times O_{\Lambda} \to \R^d$ and $\EE[\hat{T}_\xi(u,\lambda)] = T(u,\lambda)$.
\item 
$\hat{T}_x$ is path differentiable and $\hat T_x'$ is a selection of its conservative derivative $\D{\hat T_x}$   
and there exist $\sigma_1, \sigma_2, \sigma'_1, \sigma'_2 \geq 0$ such that for every $u \in \R^d$, $\lambda\in \Lambda$
\begin{gather*}
\EE[\hat T'_\xi(u,\lambda)] =  T'(u,\lambda)
\in \D{T}(u,\lambda), \\ \Var[\hat T_\xi(u,\lambda)] \leq \sigma_1 + \sigma_2\norm{u - T(u,\lambda)}^2,\quad 
\Var[\partial_1 \hat T_\xi(u,\lambda)] \leq \sigma'_1, \quad \Var[\partial_2 \hat T_\xi(u,\lambda)] \leq \sigma'_2.
\end{gather*}
where $\hat T'_x (u,\lambda) = [\partial_1 \hat T_x(u,\lambda),\partial_2 \hat T_x(u,\lambda)]$.
\end{enumerate}
\end{assumption}

\begin{remark}
The above assumptions  can be satisfied in the following situations:
(1) $G$ is nonsmooth, e.g., some proximity operator or the projection on some simple constraints, while $T$ and $\hat T_x$ are smooth (e.g., one step of gradient descent of a twice differentiable loss); (2) in view of \Cref{rm:sumphi}, when $T = \frac{1}{n}\sum_{i=1}^n \hat T_i$, 
$T' = \frac{1}{n}\sum_{i=1}^n \hat T'_i$ with $\hat T' \in \D{\hat T_i}^s$   and $\xi$ is uniformly distributed on $[n]$. 
\end{remark}

\Cref{ass:contlip} ensures that $\D{\Phi}$ obtained via the chain rule for conservative derivatives in \citep{bolte2021conservative} (see \Cref{app:SID}) is a conservative derivative of
$\Phi$ and that $\norm{\D{\Phi,1}(u,\lambda)}_{\sup} \leq q<1$.
Thus, $w(\lambda)$ is well defined and it has conservative derivatives $\D{w}^{\imp}$ and $\D{w}^{\fix}$. 
\Cref{ass:variance} is a nonsmooth generalization of the corresponding one in \citep{grazzi2021convergence,grazzi2023bilevel}. 
Finally, recalling \eqref{eq:Jimp}, if we set $\partial_2 \Phi(u,\lambda) = \partial_1 G(T(u,\lambda),\lambda) \partial_2 T(u,\lambda) + \partial_2 G(T(u,\lambda),\lambda)$ then
\begin{equation}
\label{eq:20240202a}
    \partial_2\Phi (w(\lambda),\lambda)^\top v(w(\lambda),\lambda) \in \D{w}^{\imp} (\lambda)^\top y
\end{equation}
where, for every $u \in \R^d$, $v(u,\lambda)$ is a solution of the linear system
\begin{equation}
\label{eq:20240202b}
    (I-\partial_1 T(u,\lambda)^\top \partial_1 G(T(u,\lambda), \lambda)^\top) v = y.
\end{equation}

\paragraph{Algorithm and convergence guarantees}
Our method is inspired by \eqref{eq:20240202a} and \eqref{eq:20240202b} but it uses
mini-batch estimators of $T$ and $\partial_2\Phi$. To that purpose we assume to have 
two independent sets of samples
$\bm{\xi}^{(1)}=(\xi^{(1)}_{j})_{1\leq j\leq J}$ and 
$\bm{\xi}^{(2)}=(\xi^{(2)}_{i})_{1\leq i\leq k}$, being
 i.i.d. copies of the random variable $\xi$.
Moreover, we define the path differentiable functions
\begin{align*}
     \bar{T}(u,\lambda){=} \frac{1}{J}\sum_{j=1}^{J} \hat T_{\xi^{(1)}_j}(u,\lambda), \quad
     \bar{\Phi}(u,\lambda) = G(\bar T(u,\lambda), \lambda).
\end{align*}
In fact our approach first replaces the linear system \eqref{eq:20240202b} with 
\begin{equation}
\label{eq:20240315a}
(I-\partial_1 T(w_t(\lambda),\lambda)^\top \partial_1 G(\bar{T}(w_t(\lambda),\lambda), \lambda)^\top) v = y,
\end{equation}
where the solution is in turn approximated by a stochastic sequence $(v_k)_{k \in \N}$,
which has access only to $\hat{T}_x, G$, and $w_t(\lambda)$.
Second, it outputs $\partial_2 \bar{\Phi}(w_t(\lambda), \lambda)^\top v_k$, where for any $u \in \R^d$, $\lambda \in O_\Lambda$,
\begin{align}\label{eq:partialphibar}
    \partial_2 \bar{\Phi}(u,\lambda) &{=} \partial_1 G(\bar{T}(u,\lambda),\lambda) \partial_2 \bar{T}(u,\lambda) + \partial_2 G(\bar{T}(u,\lambda),\lambda),
\end{align}
with $\bar{T}'(u,\lambda) {:=} [\partial_1 \bar{T}(u,\lambda), \partial_2 \bar{T}(u,\lambda)] {=} \frac{1}{J} \sum_{j=1}^J \hat T'_{\xi_j^{(1)}}(u,\lambda)$, which thanks to the chain rule is an element of a partial conservative derivative of $\bar{\Phi}$ (see also \Cref{app:SID}).

We now provide a general bound for the mean square error 
of an estimator of an element of the Jacobian vector product $\D{w}^{\imp}(\lambda)^\top y$,
which is agnostic with respect to the algorithms solving the fixed point equation \eqref{eq:fixedpoint}
and the linear system \eqref{eq:20240315a}. The proof (in Appendix~\ref{app:SID}) uses similar techniques as the one for the smooth case in \citep{grazzi2021convergence,grazzi2023bilevel}.
\begin{assumption}\label{ass:new}\ 
\vspace{-0.5ex} Let $\rho_\lambda\colon \N \to \R_+$, $\sigma_\lambda\colon \N \to \R_+$ 
be such that $\lim_{t\to+\infty}\rho_\lambda(t)= 0$, $\lim_{k\to+\infty}\sigma_\lambda(k)= 0$.
\begin{enumerate}[label={\rm (\roman*)}]
\item\label{ass:new_i} $(w_t(\lambda))_{t \in \N}$ is a sequence of random vectors in $\R^d$ and
\begin{equation*}
\EE[\norm{w_t(\lambda)- w(\lambda)}^2] \leq \rho_\lambda(t),
\end{equation*}
\item\label{ass:new_ii} For every $(u_1,u_2) \in \R^d\times \R^d$, $(v_k(u_1,u_2))_{k \in \N}$ is a sequence of random vectors in $\R^d$
which is independent on $(w_t(\lambda))_{t \in \N}$ and 
such that 
\begin{equation*}
\EE[\norm{v_k(u_1,u_2)- \bar{v}(u_1, u_2)}^2\,\vert\, y] \leq \norm{y}^2\sigma_\lambda(k),
\end{equation*}
where 
$\bar{v}(u_1,u_2)$ is the unique fixed point of the affine mapping
$v\mapsto \partial_1 T(u_1,\lambda)^\top \partial_1 G(u_2, \lambda)^\top v + y$.
\item\label{ass:new_iii} The r.v.~$y$ satisfies $\EE[\norm{y}^2\,\vert\, w_t(\lambda)] \leq b^2$ a.s.
\end{enumerate}
\end{assumption}

\begin{theorem}\label{th:nsidrate} Under \Cref{ass:contlip}, \ref{ass:variance}, and \ref{ass:new}, let $\kappa=(1-q)^{-1}$.
We define the estimator
\begin{equation*}
\jvp := \partial_2 \bar{\Phi}(w_t(\lambda),\lambda)^\top 
v_k\big(w_t(\lambda), \bar{T}(w_t(\lambda),\lambda)\big).
\end{equation*}
Then for every $t,k,J \in \N$, 
we have
\begin{align*}
    \EE\big[\!\gap\!\big(\jvp,& \D{w}^{\imp}(\lambda)^\top y\big)^2\big] =  b^2\times O\left( \sigma_\lambda(k)+ \kappa^4 \left(J^{-1} + \rho_\lambda(t)\right)\right).
\end{align*}
\end{theorem}

We preset the full procedure, named nonsmooth stochastic implicit differentiation (NSID), in \Cref{alg:nsid},
where the sequence $v_k$ considered in Assumption~\ref{ass:new}\ref{ass:new_ii} is generated by a simple stochastic fixed-point iteration algorithm (described in \citep{grazzi2021convergence} and recalled in \Cref{app:SID}) with step sizes $(\eta_i)_{1 \leq i \leq k}$.

\begin{algorithm}[ht]
   \caption{NSID}
   \label{alg:nsid}
\begin{algorithmic}[1]
   \STATE {\bfseries Input:} $k,J \in \N$,  $w_t(\lambda), y \in \R^d$, $\bm{\xi}^{(1)}, \bm{\xi}^{(2)}$
   \STATE $\bar T_t(\lambda) \gets \bar{T}(w_t(\lambda),\lambda)$ $\qquad$ {\small\textit{(using $\bm{\xi}^{(1)}$)}}
   \STATE $\hat \Psi\colon (v,x) \mapsto  \partial_1 \hat T_x(w_t(\lambda),\lambda)^\top \partial_1 G(\bar T_t(\lambda), \lambda)^\top v + y $
    \FOR{$i=1$ {\bfseries to} $k$}
    \STATE $v_i \gets (1- \eta_i)v_{i-1} + \eta_i \hat \Psi(v_{i-1}, \xi^{(2)}_{i})$
    \ENDFOR
    \STATE {\bfseries Return} $\jvp :=\partial_2 \bar{\Phi}(w_t(\lambda),\lambda)^\top v_k$
\end{algorithmic}
\end{algorithm}

Note that all steps can be efficiently implemented via automatic differentiation by using only vector-valued function evaluations and conservative Jacobian-vector products without the expensive computation of the full matrix derivatives. Also, using a fixed selection for the conservative derivative of $\hat T_x$ and $G$ corresponds to the standard implementation.

If $G(\cdot,\lambda)$ is the identity and $T$ is smooth, %
NSID reduces to the 
same procedure given in \citep{grazzi2023bilevel},
which also provide the bound $
O\left(\sigma_\lambda(k)+\kappa^2 J^{-1}+\kappa^4\rho_\lambda(t)\right)
$ in Theorem 7. Compared to the bound given in \Cref{th:nsidrate}, we note that the only difference is in the constant in front of the term $J^{-1}$, which we believe  may be related to the term $G$. Indeed handling a general $G$ provides an additional challenge since we do not have access anymore to an unbiased estimator of $\Phi$. However, we could overcome this issue by using different samples sequences for the two factors occuring in $\hat \Psi$. Incidentally, one of those sequences can be the one used to compute a mini-batch estimator of $\partial_2 \Phi$. Ultimately, this does not call for any additional samples compared to the smooth version, but it could worsen some constants in the bound.

Finally, we specialize the result of Theorem~\ref{th:nsidrate} to \Cref{alg:nsid}. The proof is in Appendix~\ref{app:SID}.

\begin{theorem}\label{cor:nsidrate}
Under \Cref{ass:contlip}, \ref{ass:variance}, and \ref{ass:new}\ref{ass:new_i}\ref{ass:new_iii}, let $\jvp$ be generated by Algorithm~\ref{alg:nsid} with $\eta_i = \Theta(i^{-1})$
and assume that $\rho_\lambda(t) = O(\kappa^{\alpha} t^{-1})$, with $\alpha>0$. Then
\begin{equation*}
\begin{split}
\EE\big[\!\gap\!\big(\jvp, \D{w}^{\imp}(\lambda)^\top y\big)^2\big] = 
O\left( \frac{\kappa^5}{k} {+}  \frac{\kappa^4}{J} {+} \frac{\kappa^{4+\alpha}}{t}\right).
\end{split}
\end{equation*}
Hence if $J=O(t)$, $k=O(t)$, the mean square error is  $\leq \epsilon$ after $O(\kappa^{5+\alpha} \epsilon^{-1})$ samples.
\end{theorem}

Note that the sample complexity $O(\epsilon^{-1})$ matches the performance of SGD for minimizing strongly convex and Lipschitz smooth functions \citep{bottou2018optimization}, which are a special cases of Problem \eqref{eq:fixedpointstoch}. Furthermore it is the same one that the SID algorithm by \citet{grazzi2021convergence,grazzi2023bilevel} attains when $G(v,\lambda) = v$ and $\Phi$ is Lipschitz smooth. A limitation is the choice of step-sizes $(\eta_i)$, problematic in practice.

\section{Application to Bilevel Optimization}
\label{se:bilevel}
In this section, we consider the following bilevel problem with the fixed point problem in \eqref{eq:fixedpoint} at the lower level 
\begin{equation}
\label{eq:bilevel}
    \min_{\lambda \in \Lambda}\{ E(w(\lambda),\lambda) : ~
    w(\lambda) = \Phi(w(\lambda), \lambda)\},
\end{equation}
where $E\colon \R^d \times O_\Lambda \to \R$. We will show how we can use AID-FP, ITD and NSID to approximate an element of the conservative derivative of the bilevel objective $ f(\lambda) : = E(w(\lambda),\lambda)$ and retain the same convergence rates.

In addition to the requirement that $\Phi$ satisfies Assumption~\ref{ass:lipsel}, we also make the hypothesis that $E$ satisfies the first item of same assumption with corresponding conservative derivative $\D{E} = \D{E}^s$. Therefore, applying the usual chain rule, we have that
for $* \in \{\imp, \fix\}$
\begin{align*}
    \D{f}^{*}(\lambda) &{:=} \D{E}(w(\lambda),\lambda) \, \begin{bmatrix} \D{w}^{*}(\lambda) \\ I_m\end{bmatrix}
\end{align*}
is a conservative derivative for $f$.
We also let $f_t(\lambda) : = E(w_t(\lambda), \lambda)$, where $w_t(\lambda)$ is an approximate solution for the fixed point problem.

\paragraph{Deterministic Case} The approximate derivatives
\begin{equation*}\label{eq:aiditdhg}
 \begin{aligned}
    \textbf{(BITD)}\quad \D{f_t}(\lambda) &{: =} \D{E}(w_t(\lambda),\lambda)\, \begin{bmatrix} \D{w_{t}}(\lambda) \\ I_m\end{bmatrix} \\
    \textbf{(BAID-FP)}\quad  \D{f_t}^k(\lambda) &{: =} \D{E}(w_t(\lambda),\lambda)\, \begin{bmatrix} \D{w_{t}}^{k}(\lambda) \\ I_m\end{bmatrix} 
\end{aligned}   
\end{equation*}
converge to $\D{f}^{\fix}(\lambda)$ with the same rate as ITD and AID (\Cref{th:aiditdhgrates}).

\paragraph{Stochastic Case~} We study the bilevel problem
\begin{equation}\label{eq:bilevelstoch}
\begin{gathered}
    \min_{\lambda \in \Lambda} f(\lambda) : = \EE [\hat E_\zeta(w(\lambda),\lambda)], \\
    w(\lambda) = G\big(\EE[\hat T_\xi(w(\lambda), \lambda)],\lambda\big).
\end{gathered}
\end{equation}
where $\zeta$ is a random variable.
We consider \Cref{alg:nsid-bilevel}, which additionally computes $\bar E'(w_t(\lambda),\lambda) := J_1^{-1} \sum_{j=1}^{J_1} \hat E_{\zeta^{(1)}_{j}}'(w_t(\lambda),\lambda)$, a minibatch gradient estimator of $E' \in \D{E}$, using the sequence $\bm{\zeta}^{(1)} = (\zeta^{(1)})_{1\leq j\leq J_1}$ of i.i.d. copies of $\zeta$.

\begin{algorithm}[ht]
   \caption{NSID-Bilevel}
   \label{alg:nsid-bilevel}
\begin{algorithmic}[1]
   \STATE {\bfseries Input:} $k,J_1, J_2 \in \N$ ,  $w_t(\lambda) \in \R^d$, $\bm{\xi}^{(1)}, \bm{\xi}^{(2)}, \bm{\zeta}^{(1)}$
   \STATE Compute $\bar E'(w_t(\lambda),\lambda)$ $\qquad$ {\small\textit{(using $\bm{\zeta}^{(1)}$)}}
   \STATE $y \gets \partial_1\bar{E}(w_t(\lambda),\lambda)^\top$
   \STATE $r(w_t(\lambda),\lambda) \gets \mathrm{NSID}(k, J_2, w_t(\lambda), y, \bm{\xi}^{(1)}, \bm{\xi}^{(2)})$
   \STATE {\bfseries Return} $\hat{\nabla} f(\lambda)^\top:= r(w_t(\lambda),\lambda)^\top + \partial_2 \bar E(w_t(\lambda),\lambda)$
\end{algorithmic}
\end{algorithm}
With additional mild assumptions on the variance of $\hat E$ and when $E(\cdot,\lambda)$ is Lipschitz, we recover the same convergence rates as NSID, but this time to $\D{f}^{\imp}(\lambda)$ (\Cref{th:nsidhgrates}).

\paragraph{On the convergence of the bilevel problem}
Despite these encouraging results and the fact that in the smooth case several works provide convergence rates to a stationary point of the gradient of $f$ 
\citep[and others]{ji2021bilevel,arbel2021amortized,grazzi2023bilevel}, proving such type of results or even asymptotic convergence (without rates) in our nonsmooth case is more challenging and we leave it for future work. One crucial issue is that in the analysis, the constant defined in \Cref{lm:glolip}, which we use in place of that of Lipschitz smoothness, cannot be properly controlled on the whole $\Lambda$ as required in the smooth case: it becomes arbitrarily large when $(w(\lambda), \lambda)$ approaches nondifferentiable regions of $\Phi$.

\section{Experiments}
\label{se:exp}
The experiments aim to achieve two primary goals. Firstly, we aim to empirically demonstrate the practical manifestation of distinct behaviors between AID and ITD, as outlined in the theoretical findings of Section \ref{se:detitdaid}. Emphasis is placed on aspects specific to the nonsmooth analysis. Secondly, we intend to evaluate the empirical performance of our stochastic method NSID presented in Algorithm \ref{alg:nsid}. We implement NSID by relying on PyTorch automatic differentiation for the computation of Jacobian-vector products. For AID and ITD, we use the existing PyTorch implementations\footnote{\footnotesize \url{https://github.com/prolearner/hypertorch}}.
We provide the code to reproduce our experiments at \url{https://github.com/prolearner/nonsmooth_implicit_diff}

\begin{figure}[t]
\begin{center}
\includegraphics[width=.4\columnwidth]{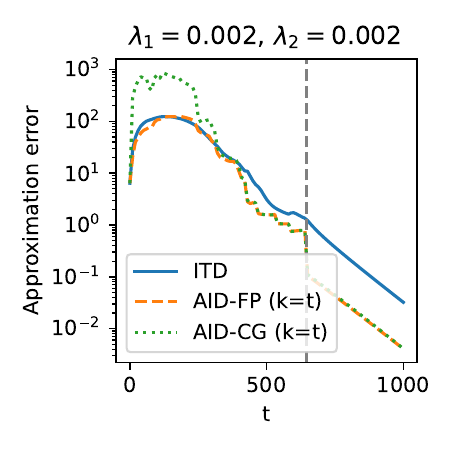}
\hspace{-.3cm}
\includegraphics[width=.4\columnwidth]{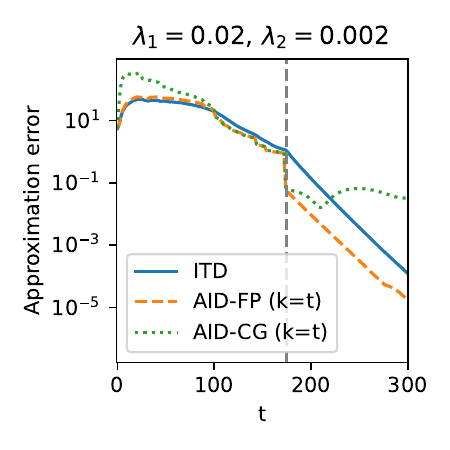}
\vspace{-.3cm}
\caption{AID vs ITD for synthetic elastic-net. $t$ corresponds to the number of steps to find an approximate fixed point and the dashed vertical line is the step where the support is identified. AID-FP converges faster than ITD; %
note that after support identification there is a wide gap between the methods, as anticipated by our theoretical bounds. AID-CG does not converge in plot on the right, probably due to  sensitivity to numerical errors.}
\label{fig:aidvsitd}
\end{center}
\end{figure}

\paragraph{Experimental Setup} We consider two problems where we are interested in approximating an element of the conservative Jacobian-vector product of the solution map $\D{w}^{\fix}(\lambda)^\top y$ for $y\in \R^d$. With a focus on bilevel optimization, we set $y$ as the gradient of the validation loss in $w_t(\lambda)$, as explained in \Cref{se:bilevel}, while to compute the approximation error we use the procedure described in \Cref{se:computingapproxerror}.\\[1ex]
{\it Elastic Net~} 
Let $(X, y) \in \R^{n \times d} \times \R^n$ be a training regression dataset.  
 The elastic net solution $w(\lambda)$ is the minimizer of the objective function 
$\frac{1}{n}\norm{X w - y}^2 + \lambda_1\norm{w}_1 + \frac{\lambda_2}{2}\norm{w}^2_2 $, where $\lambda = (\lambda_1,\lambda_2)$ are the regularization hyperparameters.\\[1ex]
{\it Data Poisoning~} 
We consider a data poisoning scenario similar to the one in \citep{xiao2015feature}, where an attacker would like to corrupt part of the training dataset by adding noise in order to decrease the accuracy of an elastic-net regularized logistic regression model after training. In particular, let $c$ be the number of classes and $(\tilde{X}, \tilde{y}) \in \R^{n' \times d} \times [c]^{n'}$ be the examples to corrupt while $(X, y) \in \R^{n \times d} \times [c]^n$ are the clean ones. Let also $\Gamma \in \R^{n' \times d}$ represent the noise and define the 
data poisoning elastic net solution as $w(\Gamma) = \argmin_{w\in \R^d} f(\Gamma, w)  + \lambda_1\norm{w}_1+ \frac{\lambda_2}{2}\norm{w}^2_2 $,
where $f(\Gamma, w) = \ell(X w, y)/2 + \ell((\tilde{X}+\Gamma)w, \tilde{y})/2$ and $\ell$ is the cross-entropy loss. A strategy to find $\Gamma$ would be by approximating an element of the conservative Jacobian-vector product $D_{w}(\Gamma)^\top y$ where $y$ is the gradient of the cross-entropy loss on an hold out set. This setting is of particular interest, since $\Gamma$ is high dimensional and hence zero-order methods like grid or random search are less appropriate.
For both settings and all considered methods, we find an approximate solution $w_t(\lambda)$ always by iterating the contraction map which describes the iterates of the deterministic iterative soft-thresholding algorithm (see e.g., \citep{combettes2005signal}). Although this may be  inefficient in the stochastic setup, it yields a fairer comparison, since both the stochastic and deterministic algorithms will have the same $w_t(\lambda)$ as input. Additional details are in the appendix.

\paragraph{AID and ITD} We consider the Elastic Net scenario and construct a synthetic supervised linear regression problem with $100$ training examples and $100$ features, of which $30$ are informative. As the fixed point map $\Phi$ we use one step of iterative soft-thresholding. The appropriate choice for the step-size guarantees that $\Phi$ is a contraction, in our case we set it equal to $2/(L+\mu + 2\lambda_2)$, where $L$ and $\mu$ are the largest and smallest eigenvalues values of $n^{-1}X^\top X$. \\
We compare ITD, AID-FP, and AID-CG a variant of AID which uses conjugate gradient to solve the linear system \citep{grazzi2020iteration}, where the vector $y$ for the Jacobian-vector product is the gradient of the square loss on a validation set, computed on the $t$-the iterate ($\nabla E(w_t(\lambda),\lambda)$ where $E$ is defined in \eqref{eq:bilevel}). In \Cref{fig:aidvsitd} we can see two runs, each one for two particular choices of $\lambda$ which highlight a wide gap in performance after support identification, i.e.\@ when  both $w_t(\lambda)$ and $w(\lambda)$ have the same non-zero elements. This was predicted by \Cref{th:itdaidrates}, since support identification coincides with $\norm{w_t(\lambda)- w(\lambda)} \leq R_\lambda$.

\paragraph{Stochastic Methods} 
We compare our stochastic method NSID (\Cref{alg:nsid}) against AID-FP and the algorithm SID in \citep{grazzi2023bilevel}. In particular, for NSID $\hat T_x$ corresponds to one step of gradient descent on a minibatch of training points, while $G$ is soft-thresholding. We implement SID by setting in NSID $G(u, \lambda) = u$ and using $\hat \Phi_\xi(u,\lambda) = G(\hat T_\xi(u,\lambda), \lambda)$ in place of $\hat{T}_\xi$. Note that although the theoretical convergence guarantee for SID do not hold  due to $\hat{\Phi}_\xi$ being biased, the performance of SID still effectively measures the impact of such bias in practice. \\
We consider both the elastic net and the data poisoning setups; see the appendix for more information.  
The results are shown in \Cref{fig:hdstoch}. For elastic net, each run corresponds to a different sampling of the covariance matrix, training points, true solution vector and minibatches used by the stochastic algorithms. For Data poisoning, each run corresponds to different sampling of the noise $\Gamma$ (sampled from a normal and then each component projected in $[-.1,.1]$) and the mini-batches used by the stochastic algorithms. For AID-FP, each epoch corresponds to one iteration, since it uses the entire dataset, while for NSID and SID the number of epochs is equal to $(k+J)(n'+n)/b$, where $b$ is the minibatch size, which we set to $10\%$ of the training set, i.e.\@ $b=(n'+ n)/10$. Note that for each point in the plots for NSID and SID, we need to start the algorithm from scratch since we increase both $k$ and $J$ simultaneously. In particular we set $k=J$ for elastic net and $J=\ceil{k/20}$ for data poisoning.

\begin{figure}[t]
    \centering
\includegraphics[width=.4\columnwidth]{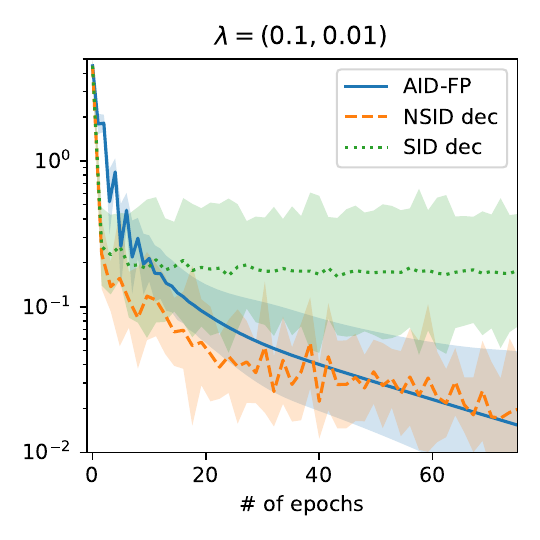}
\hspace{-.3cm}
\includegraphics[width=.4\columnwidth]{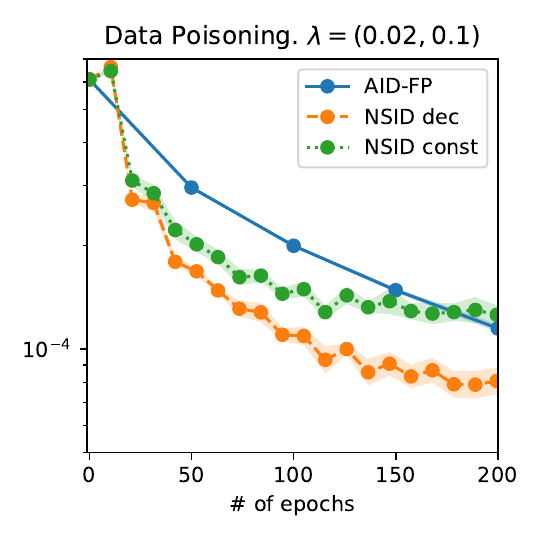}
\vspace{-.3cm}
    \caption{Stochastic implicit differentiation for elastic net (left) and data poisoning (right) with constant (const) and decreasing (dec) step sizes. Mean (solid line) and the geometric standard deviation (shaded region) of the approximation error over 10 runs. SID does not converge on elastic net for this specific choice of $\lambda$ and diverges in data poisoning (hence we do not report it), while NSID converges faster (at the beginning) than the deterministic AID-FP. Note that decreasing step-sizes provide a favorable choice.}
    \label{fig:hdstoch}
\end{figure}

\section{Conclusions}
\label{se:conclusions}
We established convergence guarantees for nonsmooth 
implicit differentiation methods. Leveraging the foundation laid by \citep{bolte2022automatic}, we developed tools facilitating the translation of results from the smooth case. This allowed us to 
provided non-asymptotic linear convergence rates for AID-FP and ITD, focusing on 
deviations from their smooth analogs. Additionally, we introduced NSID, a principled stochastic algorithm. Numerical experiments
underscored the distinctive behaviors of AID-FP and ITD, along with the good performance of NSID, which may be useful in large scale bilevel optimization problems in the future. 
Despite our results, establishing rates for solving nonsmooth bilevel problems is still challenging and we leave it for future work.

\section*{Acknowledgements~} This work was supported in part from the PNRR MUR Project PE000013 CUP J53C22003010006 “Future Artificial Intelligence Research (FAIR)“, funded by the European Union – NextGenerationEU, and EU Project ELSA under grant agreement No. 101070617.

\bibliography{ref.bib}
\bibliographystyle{icml2024}

\newpage
\appendix

\onecolumn

\begin{center}
{\Large \bf Appendices} 
\end{center}

This supplementary material is organized as follows. In App.~\ref{se:definable} we recall the notion of definable mappings. App.~\ref{app:aux} gives some auxiliary results and proof of lemmas in the main body. In App.~\ref{app:iterAID} we present the proof of \Cref{th:itdaidrates}. App.~\ref{app:SID} gives the proof of ~\Cref{th:nsidrate,cor:nsidrate}. In App.~\ref{app:BO} we address bilevel optimization. Finally, App.~\ref{app:exps} contains more information on the numerical experiments.

\section{Definable Mappings}\label{se:definable}

The concept of definable sets and functions is part of the so called tame geometry. Here we give just a very brief account (additional details can be found in \citep{bolte2021conservative}). An \emph{$o$-minimal structure on $(\R,+,\cdot)$} (`o' stands for 'ordinal') is a collection of sets $\mathcal{O}=(\mathcal{O}_p)_{p \in \N}$ such that, for each $p \in \N$,
\begin{enumerate}[label={\rm (\roman*)}]
\item $\mathcal{O}_p$ is a \emph{Boolean algebra}, meaning a nonempty family of subset of $\R^p$ which is stable by complementations and finite unions and intersections. Moreover, it contains 
the algebraic sets, that is, the sets of zeros of polynomial functions in $p$ variables.
\item $\mathcal{O}_1$ is made exactly of finite unions of intervals.
\item $A \in \mathcal{O}_p \Rightarrow\ A\times \R, \R\times A \in \mathcal{O}_{p+1}$
\item if $\pi_p\colon \R^{p+1}\to \R^p$ is the canonical projection onto the first $p$
components, then $A \in \mathcal{O}_{p+1}\ \Rightarrow\ \pi_{p}(A) \in \mathcal{O}_p$;
\end{enumerate}
Subsets of $\R^p$ which belongs to an $o$-minimal structure $\mathcal{O}$ are called \emph{definable in $\mathcal{O}$} and set-valued mappings $F\colon \R^d\multito \R^p$ are said \emph{definable in $\mathcal{O}$} if their graphs 
(as a subset of $\R^{d+p}$) is definable in $\mathcal{O}$.

There are several examples of $o$-minimal structures. The smallest one is that of real semialgebraic sets, meaning 
finite unions of sets which are solutions of a system of polynomial equations and inequalities. Here we consider  
the larger class of $\log-\exp$ structure, which additionally contains the graph of the exponential function and includes most of the functions considered in machine learning, including deep learning. 
So, in this paper definable is meant to be definable in the $\log-\exp$ $o$-minimal structure.

\section{Auxiliary Lemmas}
\label{app:aux}
\begin{lemma}[Properties of the excess]\label{lm:gapprop}
Let $\A, \B, \A^\prime, \B^\prime\subset \R^{n\times p}$ and $\C \subset \R^{d \times n}$, $\mathcal{D} \subset \R^{p \times d}$
be nonempty sets of matrices.
The following hold true: 
\begin{enumerate}[label={\rm (\roman*)}]
\item\label{eq:gapti} $\gap(\A, \C) \leq \gap(\A, \B) + \gap(\B, \C)$
\item\label{eq:gapsum} $\gap(\A + \A^\prime, \B + \B^\prime) \leq \gap(\A,  \B) + \gap(\A^\prime, \B^\prime)$
\item\label{eq:gapmul} $\gap(\C\hspace{0.1ex} \A, \C\hspace{0.1ex} \B)  \leq \norm{\C}_{\sup}\gap(\A, \B)$ and
$\gap(\A\hspace{0.1ex} \mathcal{D}, \B\hspace{0.1ex} \mathcal{D})  \leq \norm{\mathcal{D}}_{\sup}\gap(\A, \B)$  
\item\label{eq:gapincl} If $\B\subset \B^\prime$, then $\gap(\A,\B^\prime) \leq \gap(\A,\B)$.
\item\label{eq:gapinv} Suppose that $n=p$ and that all the elements in $\A$ and $\B$ are invertible. Then
\begin{equation*}
\gap(\A^{-1},\B^{-1}) \leq \norm{\A^{-1}}_{\sup}\norm{\B^{-1}}_{\sup} \gap(\A, \B).
\end{equation*}
\item\label{lm:gapprop_iv} 
Suppose that $p=p_1+p_2$ and set, for $k=1,2$ 
$\mathrm{pr}_1\colon \R^{n\times p} \to \R^{n\times p_k}$ be the canonical projections and
\begin{align*}
\A_k &= \mathrm{pr}_k(\A)= \{A_k \in \R^{n\times p_k}\,\vert\, [A_1, A_2] \in \A\}, \quad
\B_k = \mathrm{pr}_k(\B)= \{B_k \in \R^{n\times p_k}\,\vert\, [B_1, B_2] \in \A\}.
\end{align*}
Then $\gap(\A_k,\B_k) \leq \gap(\A,\B)$.

\item\label{lm:affineapp} Suppose that $p=p_1+p_2$.
Then, for all $\mathcal{X},\mathcal{Y}\subset \R^{p_1\times p_2}$, we have
\begin{gather*}
    \norm{\A(\mathcal{X})}_{\sup} \leq \norm{\A_1}_{\sup}\norm{\mathcal{X}}_{\sup} + \norm{\A_2}_{\sup}, \\
    \gap(\A(\mathcal{X}), \A(\mathcal{\mathcal{Y}})) \leq \norm{\A_1}_{\sup} \gap(\mathcal{X}, \mathcal{Y}), \qquad
    \gap(\A(\mathcal{X}), \B(\mathcal{X})) \leq (1+\norm{\mathcal{X}}_{\sup} )\gap(\A,\B)
\end{gather*}
where we recall that $\A(\mathcal{X}) = \{A_1 X +A_2 \given [A_1,A_2] \in \A, X \in \mathcal{X}\}$.

\end{enumerate}
\end{lemma}
\begin{proof}
In the following when $A$ is a matrix and $\B$ is a set of matrices we set
$d(A, \B) := \inf_{B \in \B} \norm{A - B}$, which is the distance from $A$ to the set $\B$.

\ref{eq:gapti}: Let $A\in \A$ and $B \in \B$. Then
\begin{align*}
(\forall\,C\in \C)\quad
&d(A,\C) \leq \norm{A-C} \leq \norm{A - B} + \norm{B- C}\\
&\implies 
d(A,\C) - \norm{A - B} \leq \norm{B- C}.
\end{align*}
Thus
\begin{equation*}
d(A,\C) - \norm{A - B} \leq d(B,\C) \leq \gap(\B,\C)
\end{equation*}
and hence
\begin{equation*}
(\forall\,B \in \B)\quad d(A, \C) - \gap(\B,\C) \leq \norm{A-B}.
\end{equation*}
So, $d(A, \C) - \gap(\B,\C) \leq d(A,\B)\leq \gap(\A,\B)\implies d(A, \C) \leq \gap(\B,\C) + d(\A,\B)$.
Taking the sup in $A \in \A$ the statement follows.

\ref{eq:gapsum}: Let $A\in \A, A^\prime \in \A$. Then, 
\begin{align*}
(\forall\, B\in \B)
(\forall\, B^\prime\in \B^\prime)\quad
d(A+A^\prime, \B+\B^\prime) &\leq \norm{(A+A^\prime)-(B+B^\prime)}\\ 
&\leq \norm{A-B} + \norm{A^\prime-B^\prime}.
\end{align*}
Thus,
\begin{equation*}
d(A+A^\prime, \B+\B^\prime) \leq d(A, \B) + d(A^\prime, \B^\prime) \leq \gap(\A,\B) + \gap(\A^\prime,\B^\prime).
\end{equation*}
Since $A$ and $A^\prime$ are arbitrary in $\A$ and $\A^\prime$ respectively, the statement follows.

\ref{eq:gapmul}: Let $A\in \A$, $B \in \B$ and $C \in C$. Then
\begin{equation*}
d(C A, \C\B) \leq \norm{C A - C B} \leq \norm{C} \norm{A - B} \leq \norm{\C}_{\sup} \norm{A- B}.
\end{equation*}
Taking the infimum over $B \in \B$ we get
\begin{equation*}
d(C A, \C\B) \leq \norm{\C}_{\sup} \inf_{B \in \B}\norm{A- B} \leq \norm{\C}_{\sup} \gap(\A, \B).
\end{equation*}
Now, taking the supremum over $C\in \C$ and $A \in \A$, the statement follows. A similar proof can be applied for the other case.

\ref{eq:gapinv}: Let $A \in \A$ and $B \in \B$. Then $A^{-1} - B^{-1} = A^{-1} (B- A) B^{-1}$
and hence
\begin{equation*}
\norm{A^{-1} - B^{-1}} \leq \norm{A^{-1}} \norm{A- B} \norm{B^{-1}} \leq \norm{\A^{-1}}_{\sup}\norm{\B^{-1}}_{\sup}
\norm{A- B}.
\end{equation*}
Thus
\begin{align*}
\inf_{B \in \B}\norm{A^{-1} - B^{-1}} &\leq \norm{\A^{-1}}_{\sup}\norm{\B^{-1}}_{\sup}
\inf_{B \in \B}\norm{A- B}\\ 
&\leq\norm{\A^{-1}}_{\sup}\norm{\B^{-1}}_{\sup}\gap(\A,\B).
\end{align*}
Taking the supremum in $A \in \A$, the statement follows.

\ref{lm:gapprop_iv}: We first note that if $A = [A_1, A_2] \in \R^{d\times (p_1+p_2)}$ we have
\begin{equation*}
\norm{A_1} = \sup_{\norm{x}\leq 1} \norm{A_1 x} = \sup_{\norm{(x,0)}\leq 1} \Big\lVert [A_1 A_2]\begin{bmatrix} x\\0\end{bmatrix}\Big\rVert \leq  \norm{A}
\end{equation*}
and similarly $\norm{A_2}\leq \norm{A}$. Now let $A_1 \in \A_1$ and $B = [B_1 B_2] \in \B$. Then there exists 
$A_2$ such that $A=[A_1 A_2] \in \A$ and hence
\begin{equation*}
d(A_1, \B_1) \leq \norm{A_1 - B_1} \leq \norm{A - B}.
\end{equation*}
Since the above inequality holds for every $B \in \B$ we have
\begin{equation*}
d(A_1, \B_1) \leq \inf_{B \in \B} \norm{A - B} \leq \gap(\A,\B)
\end{equation*}
which in turns holds for every $A_1 \in \A_1$. Thus, taking the supremum in $A_1 \in \A_1$ the statement follows with $k=1$.
The other case is proved in the same manner.

\ref{lm:affineapp}:
For the first inequality we have
\begin{align*}
    \norm{\A(\mathcal{X})}_{\sup} &= \sup_{A \in \A, X\in \mathcal{X}} \norm{A_1 X + A_2} \\
    &\leq \sup_{A \in \A, X\in \mathcal{X}} \left(\norm{A_1} \norm{X} + \norm{A_2}\right) \\
    &\leq \sup_{A \in \A}\norm{A_1} \sup_{A \in \mathcal{X}}\norm{X} + \sup_{A' \in \A}\norm{A'_2} = \norm{\A_1}_{\sup}\norm{\mathcal{X}}_{\sup} + \norm{\A_2}_{\sup}.
\end{align*}
For the second inequality we have
\begin{align*}
    \gap(\A(\mathcal{X}), \A(\mathcal{Y})) &= \sup_{A \in \A, X \in \mathcal{X}} \inf_{A' \in \A, Y \in \mathcal{Y}} \norm{A_1 X - A_2 - A'_1 Y + A'_2} \\
    &\leq \sup_{A \in \A, X \in \mathcal{X}} \inf_{Y \in \mathcal{Y}} \norm{A_1 (X - Y)} \\
    &\leq \sup_{A \in \A} \norm{A_1} \sup_{X \in \mathcal{X}} \inf_{Y \in \mathcal{Y}} \norm{X - Y} 
    = \norm{\A_1}_{\sup}\gap(\mathcal{X},\mathcal{Y}).
\end{align*}
For the third inequality  we have
\begin{align*}
    \gap(\A(\mathcal{X}), \B(\mathcal{X})) &= \sup_{A \in \A, X \in \mathcal{X}} \inf_{B \in \B, X' \in \mathcal{X}} \norm{A_1 X - A_2 - B_1 X' + B_2} \\
    &\leq \sup_{A \in \A, X \in \mathcal{X}} \inf_{B \in \B} \norm{(A_1 -B_1)X - A_2 + B_2} \\
    &\leq \sup_{A \in \A, X \in \mathcal{X}} \inf_{B \in \B} \left(\norm{A_1 -B_1}\norm{X} +\norm{A_2 - B_2} \right)\\
    &\leq \sup_{A \in \A, X \in \mathcal{X}} \inf_{B \in \B} \left(\norm{A -B}\norm{X} +\norm{A - B} \right)\\
    &\leq \sup_{X \in \mathcal{X}} (1 + \norm{X}) \sup_{A \in \A} \inf_{B \in \B} \norm{A - B} 
    =(1+\norm{\mathcal{X}}_{\sup})\gap(\A,\B).
\end{align*}

The proof is complete.
\end{proof}

We now recall the following result from \citep{bolte2022automatic} (Lemma~4 in the Appendices),
which is stated in a slightly more general form.
\begin{theorem}\label{thm:20240130a}
Let $F\colon U\subset\R^p\to \R^d$ be a continuous selection of the definable 
Lipschitz smooth mappings $F_1, \dots, F_r \colon U\subset\R^p\to \R^d$.
Let $L_i$ be the Lipschitz constant of $F^\prime_i$ and set $L = \max_{1 \leq i \leq r} L_i$.
Then, for any $x \in U$ there exists $R_{x}>0$ such that
\begin{equation*}
\forall\, x^\prime \in U\ \text{with}\ \norm{x^\prime- x} \leq R_{x}\colon\ \gap(\D{F}^s(x^\prime), \D{F}^s(x)) \leq L \norm{x^\prime - x}.
\end{equation*}
\end{theorem}
\begin{proof}
Similarly to \citep{bolte2022automatic} we define
\begin{equation*}
g\colon \left]0,+\infty\right[\multito [r]\quad\text{such that}\ g(\rho) = I_F(B_\rho(x)),
\end{equation*}
where $B_\rho(x)$ is the closed ball of radius $\rho>0$ centered at $x$. Now, we note that $g$
is the composition of the maps
\begin{equation*}
\varphi\colon \left]0,+\infty\right[\multito \R^p\colon \rho\multito B_\rho(x),
\quad\text{and}\quad I_F\colon \R^p \multito [r].
\end{equation*}
The first one is clearly semialgebraic and hence definable and the second map
 is definable by definition (since the $F_i$'s are definable, it is easy to see that $F$ is definable if and only if $I_F$ is definable). Thus, being $g$ composition of definable set-valued mappings it is definable. Then, for every $I\subset [r]$, we have that the set
 $[g=I] = \{\rho \in \left]0,+\infty\right[ \,\vert\, g(\rho)=I\}$ is definable and setting $\mathcal{J} = \{ g(\rho) \,\vert\, \rho \in \left]0,+\infty\right[\}\subset 2^{[r]}$, we have that $([g=I])_{I \in \mathcal{J}}$, is a finite partition of $\left]0,+\infty\right[$ made of definable sets of the real line.
Thus, each one of them must be finite unions of disjoints intervals, which shows that $g$ is piecewise constant. It follows that
there exists $R_x>0$ and $I\subset [r]$ such that for every $\rho \in \left]0,R_x\right]$ $g(\rho)=I$. The proof continues as in 
Lemma~4 in \citep{bolte2022automatic}.
\end{proof}

\begin{proof}[\proofname{} of \Cref{lm:glolip}]
Let $x \in U$. Let  $\Delta_r = \{\alpha \in \R_+^r \,\vert\, \sum_{i=1}^r \alpha_i = 1\}$ be the unit simplex of $\R^r$
and $\Delta^x_r = \{\alpha \in \Delta_r \,\vert\, \forall\, i \in [r]\setminus I_F(x)\colon \alpha_i=0\}$ 
(which is essentially the unit simplex of $\R^{I_F(x)}$). Set $\A = \mathrm{co}(\{\partial F_i(x) \,\vert\, i \in I_F(x^\prime)\})$. Then, using the property of the excess in \Cref{lm:gapprop}\ref{eq:gapti}
\begin{equation*}
\gap(\D{F}^s(x^\prime), \D{F}^s(x))
\leq \underbrace{\gap(\D{F}^s(x^\prime), \A)}_{(1)} + \underbrace{\gap(\A, \D{F}^s(x))}_{(2)}.
\end{equation*}
We will bound the two terms $(1)$ and $(2)$ separately.
We recall that
\begin{equation*}
\D{F}^s(x^\prime) = \Conv(\{F^\prime_i(x^\prime) \,\vert\, i \in I_F(x^\prime)\})
\quad\text{and}\quad
\D{F}^s(x) = \Conv(\{F^\prime_i(x) \,\vert\, i \in I_F(x)\}).
\end{equation*}
Then
\begin{align*}
(1) &= \sup_{\alpha \in \Delta_r^{x^\prime}} \inf_{\beta \in \Delta_r^{x^\prime}} 
\Big\lVert \sum_{i \in I(x^\prime)} \alpha_i F^\prime_i(x^\prime) - \sum_{i \in I(x^\prime)} \beta_i F^\prime_i(x)\Big\rVert \\[1ex]
&\leq \sup_{\alpha \in \Delta_r^{x^\prime}}
\Big\lVert \sum_{i \in I(x^\prime)} \alpha_i ( F^\prime_i(x^\prime)- F^\prime_i(x))\Big\rVert\\[1ex]
&\leq \sup_{\alpha \in \Delta_r^{x^\prime}}\sum_{i \in I(x^\prime)} \alpha_i \norm{ F^\prime_i(x^\prime)- F^\prime_i(x)}
\leq \sup_{\alpha \in \Delta_r^{x^\prime}}\sum_{i \in I(x^\prime)} \alpha_i L \norm{x- x^\prime} = L \norm{x- x^\prime}.
\end{align*}
Moreover,
\begin{align*}
(2) &= \sup_{\alpha \in \Delta_r^{x^\prime}} \inf_{\beta \in \Delta_r^{x}} 
\Big\lVert \sum_{i \in I(x^\prime)} \alpha_i F^\prime_i(x) - \sum_{i \in I(x)} \beta_i F^\prime_i(x)\Big\rVert
= \sup_{\alpha \in \Delta_r^{x^\prime}} \inf_{\beta \in \Delta_r^{x}} 
\Big\lVert \sum_{i=1}^r (\alpha_i - \beta_i) F^\prime_i(x)\Big\rVert\\[1ex]
&\leq \sup_{\alpha \in \Delta_r} \inf_{\beta \in \Delta_r^{x}} 
\Big\lVert \sum_{i=1}^r (\alpha_i - \beta_i) F^\prime_i(x)\Big\rVert=:(*).
\end{align*}
Now we note that
\begin{equation*}
\varphi(\alpha, \beta) = \Big\lVert \sum_{i=1}^r (\alpha_i - \beta_i) F^\prime_i(x)\Big\rVert + \iota_{\Delta_r}(\alpha) + \iota_{\Delta_r^x}(\beta)
\end{equation*}
is jointly convex, hence $\alpha\mapsto \inf_{\beta} \varphi(\alpha, \beta)$ is convex and its maximum
is achieved at the vertices of $\Delta_r$. Thus, if we set $e_i = (\delta_j^i)_{1 \leq j \leq r}$ the canonical basis of $\R^r$,
we have
\begin{align*}
(*) &= \max_{1 \leq i \leq r} \inf_{\beta \in \Delta_r^x} \Big\lVert \sum_{j=1}^r (\delta^i_j - \beta_j) F^\prime_j(x)\Big\rVert
 = \max_{1 \leq i \leq r} \inf_{\beta \in \Delta_r^x} \Big\lVert  F^\prime_i(x)- \sum_{j=1}^r\beta_j F^\prime_j(x)\Big\rVert\\[1ex]
& \leq \max_{1 \leq i \leq r} \inf_{j \in I(x)} \Big\lVert F^\prime_i(x)- F^\prime_j(x)\Big\rVert = M_x.
\end{align*}
In the end
\begin{equation*}
\gap(\D{F}^s(x^\prime), \D{F}^s(x)) \leq M_x + L \norm{x^\prime- x}.
\end{equation*}
Now, let $R_x>0$ be as in Theorem~\ref{thm:20240130a}. Then
if $\norm{x^\prime - x}>R_x$ we have $\norm{x^\prime-x}/R_x>1$ and hence
\begin{equation*}
\gap(\D{F}^s(x^\prime), \D{F}^s(x)) \leq \frac{M_x}{R_x}\norm{x^\prime - x} + L \norm{x^\prime- x} 
= \bigg( \frac{M_x}{R_x}+ L \bigg)\norm{x^\prime - x},
\end{equation*}
otherwise, if $\norm{x - x^\prime} \leq R_x$, then by Theorem~\ref{thm:20240130a}, we have
\begin{equation*}
\gap(\D{F}^s(x^\prime), \D{F}^s(x)) \leq L \norm{x^\prime - x} \leq
\bigg( \frac{M_x}{R_x}+ L \bigg)\norm{x^\prime - x}.
\end{equation*}
The statement follows.
\end{proof}

\begin{lemma}\label{lm:Jimpfixbound}
Under \Cref{ass:lipsel}\ref{ass:lipsel_iii}, for every $(u,\lambda) \in \R^p\times \Lambda$, 
\begin{align*}
    \norm{(I-\D{\Phi,1}(u,\lambda))^{-1}}_{\sup} \leq \frac{1}{1-q}, \qquad 
    \norm{\D{w}^{\imp} (\lambda)}_{\sup} \leq \norm{\D{w}^{\fix} (
\lambda)}_{\sup} \leq \frac{\norm{\D{\Phi, 2}(w(\lambda),\lambda)}_{\sup}}{1-q}.
\end{align*}
\end{lemma}
\begin{proof}
As for the first inequality, we recall that for any matrix $A$ such that $\norm{A} \leq q < 1$, we have 
$(I - A)^{-1} = \sum_{n=0}^{\infty} A^{n}$ and hence $\norm{I-A} \leq \sum_{n=0}^{+\infty} \norm{A}^n \leq 
\sum_{n=0}^{+\infty} q^n = 1/(1-q)$. Thus,
if we let $\A = \D{\Phi}(u,\lambda)$ we have that
\begin{align*}
 \norm{(I - \A_1)^{-1}}_{\sup} = \sup_{A_1 \in \A_1} \norm{(I- A_1)^{-1}} \leq \frac{1}{1-q}.
\end{align*}
The second inequality holds since $\D{w}^{\imp} (\lambda) \subset \D{w}^{\fix} (\lambda)$. For the last inequality we note that if we let $\B = \D{\Phi}(w(\lambda),\lambda)$, it follows from the definition of $\D{w}^{\fix}$ that
\begin{equation*}
\D{w}^{\fix}(\lambda) = \B (\D{w}^{\fix}(\lambda)).
\end{equation*}
Thus, applying \Cref{lm:gapprop}\ref{lm:affineapp}
and recalling that $\norm{\D{\Phi,1}(w(\lambda),\lambda)}_{\sup} \leq q < 1$ 
we have
\begin{align*}
    \norm{\D{w}^{\fix} (\lambda)}_{\sup} &\leq \norm{\D{\Phi, 1}(w(\lambda),\lambda)}_{\sup}\norm{\D{w}^{\fix} (\lambda)}_{\sup} + \norm{\D{\Phi, 2}(w(\lambda),\lambda)}_{\sup} \\
    &\leq q\norm{\D{w}^{\fix} (\lambda)}_{\sup} + \norm{\D{\Phi, 2}(w(\lambda),\lambda)}_{\sup}
\end{align*}
which implies the last inequality, after rearranging the terms.
\end{proof}

\section{Iterative and Approximate Implicit Differentiation}
\label{app:iterAID}

Note that if $\kappa = 1/(1-q)$, then $q^t = \exp(-\log(1/q)t) \leq \exp(-t/\kappa)$.

\begin{proof}[\proofname{} of \Cref{th:itdaidrates}] 
Let $\lambda \in \Lambda$ and $t \in \N$, $t\geq 1$. 
For the sake of brevity, we set 
\begin{align*}
b_{\lambda,t} &= \big(\norm{\D{w}^{\fix}(\lambda)}_{\sup} + 1\big)C_\lambda(w_t(\lambda)), \\
 \A_t &= \D{\Phi}(w_{t}(\lambda),\lambda), \quad \A_{t,1} = \D{\Phi, 1}(w_{t}(\lambda),\lambda), \quad \B = \D{\Phi}(w(\lambda),\lambda), \end{align*}
where $C_\lambda$ is defined in Lemma~\ref{lm:glolipphi}.
We recall that
\begin{align*}
    \D{w_{t}}(\lambda) &= \A_{t-1}( \D{w_{t-1}}(\lambda) ), \qquad
    \D{w}^{\fix}(\lambda) =\B(\D{w}^{\fix}(\lambda)).
\end{align*}
We also recall that $\delta_\lambda(t) = \indic\{\norm{w_t(\lambda) - w(\lambda)} > R_\lambda\}\in \{0,1\}$ and hence
\begin{equation*}
    C_\lambda(w_t(\lambda)) = L + \frac{M_\lambda}{R_\lambda} \delta_\lambda(t).
\end{equation*}

ITD \eqref{eq:itdbound}:
Let $\Delta^\prime_t : = \gap(\D{w_t}(\lambda), \D{w}^{\fix}(\lambda))$. 
Using the properties in \Cref{lm:gapprop}\ref{eq:gapti}\ref{lm:affineapp} we have
\begin{align*}
    \Delta'_t &= \gap(\A_{t-1}( \D{w_{t-1}}(\lambda) ) , \B( \D{w}^{\fix}(\lambda) )) \\
    &\leq \gap(\A_{t-1}( \D{w_{t-1}}(\lambda) ) , \A_{t-1}( \D{w}^{\fix}(\lambda) )) 
    + \gap(\A_{t-1}( \D{w}^{\fix}(\lambda) ) , \B( \D{w}^{\fix}(\lambda) ))  \\
    &\leq \norm{\A_{t-1,1}}_{\sup} \Delta^\prime_{t-1} +(1+\norm{\D{w}^{\fix} (\lambda)}_{\sup}) \gap(\A_{t-1},\B)  \\
    &\leq q \Delta'_{t-1} + b_{\lambda,t-1} \Delta_{t-1},
\end{align*}
where for the last inequality we used that for any $u \in \R^d$, $\norm{\D{\Phi,1}(u,\lambda)} < q$ and \Cref{lm:glolipphi}. By unrolling the recursive inequality and using the inequality $\Delta_i \leq q^i\Delta_0$ we obtain
\begin{align*}
     \Delta'_t &\leq q^t \Delta'_0 +  \sum_{i=0}^{t-1} q^{t-1-i} b_{\lambda,i}\Delta_i 
     \leq q^t \Delta'_{0} +  q^{t-1} \Delta_0 \sum_{i=0}^{t-1} b_{\lambda,i}   \\
     & \leq q^t \norm{\D{w}^{\fix}(\lambda)}_{\sup} + \Delta_0  q^{t-1} t (\norm{\D{w}^{\fix}(\lambda)}_{\sup} + 1)(L+M_\lambda R^{-1}_\lambda  \bar{\delta}_\lambda(t)),
\end{align*}
where, in the last inequality, we used $\Delta'_0 \leq  \norm{\D{w}^{\fix}(\lambda)}_{\sup}$ and the definitions of $\bar{\delta}_\lambda(t)$ and $C_\lambda(w_t(\lambda))$. Applying \Cref{lm:Jimpfixbound}, factoring out $t$ and using the definition of $B_\lambda$ gives the final result.

AID-FP \eqref{eq:aidfpbound}: In this case we have
\begin{equation*}
    \D{w_{t}}^{k}(\lambda) = \A_t( \D{w_t}^{k-1}(\lambda)).
\end{equation*}
Set $\Delta'_k : = \gap(\D{w_{t}}^{k}(\lambda), \D{w}^{\fix}(\lambda))$. Then
using again \Cref{lm:gapprop}\ref{eq:gapti}\ref{lm:affineapp} we have
\begin{align*}
    \Delta'_k &= \gap(\A_t (\D{w_t}^{k-1}(\lambda)), \B (\D{w}^{\fix}(\lambda))) \\
    &\leq \gap(\A_t (\D{w_t}^{k-1}(\lambda)), \A_t (\D{w}^{\fix}(\lambda))) + \gap(\A_t (\D{w}^{\fix}(\lambda)), \B (\D{w}^{\fix}(\lambda)))  \\
    &\leq \norm{\A_{t,1}}_{\sup} \Delta'_{k-1} +(1+\norm{\D{w}^{\fix}(\lambda)}_{\sup}) \gap(\A_t,\B) \\
    &\leq q \Delta'_{k-1} + b_{\lambda,t} \Delta_t,
\end{align*}
where for the last inequality we used Assumption~\ref{ass:lipsel}\ref{ass:lipsel_iii} and \Cref{lm:glolipphi}. By unrolling the inequality recursion we obtain
\begin{align*}
     \Delta'_k \leq q^k \Delta'_0 + b_{\lambda,t} \Delta_t \sum_{i=0}^{k-1} q^i = q^k \norm{\D{w}^{\fix}(\lambda)}_{\sup} + b_{\lambda,t} \frac{1-q^k}{1-q}\Delta_t .
\end{align*}
Applying \Cref{lm:Jimpfixbound} and using the definition of $b_{\lambda,t}$, $\Ll$ and $\delta_\lambda$ gives the final result. 

For the final comment, if $w_t(\lambda) = \Phi(w_{t-1}(\lambda), \lambda)$, due to the contraction property of $\Phi$, $\Delta_t = q \Delta_{t-1} < \Delta_{t-1}$ and there exist $\tau_\lambda \in \{0,\dots,t'_\lambda\}$ with $t'_\lambda := \ceil{\log (\Delta_0/R_\lambda)/\log(1/q)}$, such that $\norm{w_{\tau_\lambda}(\lambda)- w(\lambda)} \leq R_\lambda$, and if $\tau_\lambda \neq 0$, $\norm{w_{\tau_\lambda-1}(\lambda)- w(\lambda)} > R_\lambda$. Thus, for every $i \in \N$ $\delta_\lambda(w_i) = \indic\{ i < \tau_\lambda \}$ and therefore for every $t$,  $\delta_\lambda(w_t) \leq \bar{\delta}_\lambda(t) \leq 1$.
\end{proof}

\section{Stochastic Implicit Differentiation}
\label{app:SID}

For simplicity let for every $u \in \R^d,\lambda \in O_\Lambda$
\begin{equation*}
    \Psi(u, \lambda) = (T(u,\lambda), \lambda), \qquad \D{\Psi}(u,\lambda) = \left\{\begin{bmatrix}
        C_1 & C_2 \\
         0 & I_m
    \end{bmatrix} \,\Big\vert\, [C_1, C_2] \in \D{T}(u,\lambda)  \right\} 
\end{equation*}
From Lemma 3 in \citep{bolte2021conservative} and \Cref{ass:contlip}\ref{ass:lipgt} it follows that $D_{\Psi}$ is a conservative derivative of $\Psi$. 
Moreover, we can write $\Phi(u,\lambda) = G(\Psi(u,\lambda))$ and thanks to the chain rule of conservative derivatives we have that
\begin{equation}\label{eq:dphi}
    \D{\Phi}(u,\lambda) := \D{G}(\Psi(u,\lambda))\D{\Psi}(u,\lambda) = \D{G}(T(u,\lambda), \lambda) \begin{bmatrix}
        \D{T}(u,\lambda) \\
        0 \ \  I_m
    \end{bmatrix}
\end{equation}
is a conservative derivative for $\Phi$. Furthermore, if \Cref{ass:contlip}\ref{ass:contgt} is satisfied, then $\norm{\D{\Phi,1}(u,\lambda)}_{\sup} \leq q < 1$ and $\D{w}^{\fix}$ and $\D{w}^{\imp}$ in \eqref{eq:Jfix} and \eqref{eq:Jimp} are well defined and conservative derivatives of $w$.
Similarly, a conservative derivative of $\bar{\Phi}$ can be obtained as
\begin{equation}\label{eq:derphibar}
    \D{\bar{\Phi}}(u,\lambda) : = \D{G}(\bar{T}(u,\lambda), \lambda) \begin{bmatrix}
        \D{\bar{T}}(u,\lambda) \\
        0 \ \  I_m
    \end{bmatrix}, \quad with \quad \D{\bar{T}}(u,\lambda) = \frac{1}{J}\sum_{j=1}^{J} \D{\hat{T}_{\xi^{(1)}_j}}(u,\lambda).
\end{equation}
Note that $\partial_2 \bar{\Phi}(u,\lambda)$ as defined in \eqref{eq:partialphibar} is an element of $\D{\bar{\Phi},2}$.

The following result is similar to \Cref{lm:glolipphi} and follows directly from \Cref{lm:glolip}.
The only difference is that the constants are majorized so to be independent on $u$. This is done only to simplify the analysis.

\begin{lemma}\label{lm:lipgt}
Under \Cref{ass:contlip}, for every $\lambda \in \Lambda$, there exist $R_{G,\lambda}, R_{T,\lambda} > 0$ such that for every $u \in \R^d$ and
\begin{align*}
\gap(\D{G}(u,\lambda),\D{G}(T(w(\lambda),\lambda),\lambda) &\leq \LG\norm{u- T(w(\lambda),\lambda)} \\    \gap(\D{T}(u,\lambda),\D{T}(w(\lambda),\lambda)) &\leq \LT\norm{u- w(\lambda)},
\end{align*}
where $\LG:= L_G +M_{G,\lambda}/R_{G,\lambda}$, $\LT := L_T +M_{T,\lambda}/R_{T,\lambda}$, with
\begin{align*}
 M_{T,\lambda} &:= \max_{i\in \{1,\dots,r\}}\min_{j \in I_T(w(\lambda),\lambda)} \norm{T'_i(w(\lambda),\lambda)- T'_j(w(\lambda),\lambda)}   \\
 M_{G,\lambda} &:= \max_{i\in \{1,\dots,r\}}\min_{j \in I_G(T(w(\lambda),\lambda),\lambda)} \norm{G'_i(T(w(\lambda),\lambda),\lambda)- G'_j(T(w(\lambda),\lambda),\lambda)}
\end{align*}
and
$L_T, L_G$ satisfying \Cref{ass:lipsel}\ref{ass:lipsel_i}.
\end{lemma}

We now present the proof of \Cref{th:nsidrate}.

\begin{proof}[\proofname{} of \Cref{th:nsidrate}]
Set, for the sake of brevity, $\hz=(w_t(\lambda),\lambda)$
and $\sz = (w(\lambda),\lambda)$. We also set $v_k = v_k(w_t(\lambda), \bar{T}_t(\lambda))$, $\bar{v} =\bar{v}(w_t(\lambda), \bar{T}_t(\lambda))$, and
$a_\lambda = \norm{w(\lambda) -T(\sz)}$.
From Assumption~\ref{ass:variance} on the variance of $\hat T$ and since $T(\cdot,\lambda)$ is $1$-Lipschitz we have
\begin{align}\label{eq:varhatt}
\nonumber    \Var[\hat T_\xi(\hz)\,\vert\, w_t(\lambda),y] &\leq \sigma_1 +\sigma_2 \norm{w_t(\lambda) \mp w(\lambda) - T(\hz) \pm T(\sz)}^2\\
    & \leq   \sigma_1 +3\sigma_2(2\Delta_t^2 + a_\lambda^2) =: b_{\lambda}(\Delta_t^2).
\end{align}
Now, recall that $G'(\bar{T}(z_t),\lambda) = [\partial_1 G(\bar{T}(z_t),\lambda), \partial_2 G(\bar{T}(z_t),\lambda)]$, $T'(z_t) = [\partial_1 T(z_t), \partial_2 T(z_t)]$ and set
\begin{align*}
    B &:= [B_1, B_2] = \argmin_{B' \in \D{G}(T(z),\lambda)} \norm{G'(\bar{T}(z_t),\lambda)- B'} \\
    C &:= [C_1, C_2] =  \argmin_{C' \in \D{T}(z)} \norm{T'(z_t)- C'} 
\end{align*} 
with $B_1, C_1 \in \R^{d\times d}$, $B_2, C_2 \in \R^{d\times m}$, which is valid since the $\argmin$ is over compact convex sets. 
Then, recalling the definition of excess and applying \Cref{lm:lipgt} we have that for $j=1,2$
\begin{equation}\label{eq:boundab}
\begin{aligned}
    \norm{\partial_j G(\bar{T}(z_t),\lambda)- B_j} &\leq \norm{G'(\bar{T}(z_t),\lambda)- B} = \gap(G'(\bar{T}(z_t),\lambda), \D{G}(T(z),\lambda)) \leq \LG \norm{\bar{T}(z_t) - T(z)}, \\
    \norm{\partial_j T(z_t)- C_j} &\leq \norm{T'(z_t)- C} = \gap(T'(z_t), \D{G}(z)) \leq \LG \norm{z_t - z}.
\end{aligned}
\end{equation}
Let also $A := [A_1, A_2]$ with $A_1 := B_1C_1$ and $A_2:= B_1C_2 + B_2$.
Since $B \in \D{G}(T(\sz),\lambda)$, $C \in \D{T}(\sz)$ we have that $A \in D_{\Phi}(z)$ (from the definition of $D_{\Phi}$ in \eqref{eq:dphi}) and consequently that $(I-A_1)^{-1}A_2 \in \D{w}^{\imp}(\lambda)$. Hence, recalling the definition of excess we can write 
\begin{equation*}
    \gap(\partial_2 \bar{\Phi}(\hz)^\top v_k, \D{w}^{\imp}(\lambda)^\top y) \leq \norm{\partial_2 \bar{\Phi}(\hz)^\top v_k - A_2^\top(I-A_1)^{-\top}y}.
\end{equation*}
To prove the result, it is therefore sufficient to appropriately control the distance to a particular element of $\D{w}^{\imp}(\lambda)^\top y$, namely $A_2^\top(I-A_1)^{-\top}y$, which is a random variable depending on $y, w_t(\lambda), \bm{\xi}^{(1)}$ (from the definition of $B$ and $C$).
We have the following error decomposition
\begin{equation*}
 \begin{aligned}
    \gap(\partial_2 \bar{\Phi}(\hz)^\top v_k, \D{w}^{\imp}(\lambda)^\top y)
    &\leq \norm{\partial_2 \bar{\Phi}(\hz)^\top v_k - A_2^\top  v_k} + \norm{A_2^\top  v_k - A_2^\top (I- A_1^\top)^{-1} y})\\
    &\leq \norm{\partial_2 \bar{\Phi}(\hz) - A_2}\norm{v_k}
    + \norm{\D{\Phi,2}(z)}_{\sup}
    \norm{v_k- (I- A_1^\top)^{-1} y},
\end{aligned}   
\end{equation*}
where we used that $A_2 \in \D{\Phi,2}(z)$.
Hence, squaring and taking the conditional expectation of both sides yields
\begin{equation}\label{eq:errordecom_i}
\begin{aligned}
    \EE [\gap(\partial_2 \bar{\Phi}(\hz)^\top v_k, \D{w}^{\imp}(\lambda)^\top &y)^2]\,\vert\, w_t(\lambda),y, \bm{\xi}^{(1)}] \leq  \underbrace{2\EE[\norm{v_k}^2\,\vert\, w_t(\lambda),y, \bm{\xi}^{(1)}]\cdot \norm{\partial_2 \bar{\Phi}(\hz) - A_2}^2}_{(1)} \\[1ex]
    &\quad+\underbrace{2\norm{\D{\Phi,2}(z)}_{\sup}^2\EE [\norm{v_k- (I- A_1^\top)^{-1} y}^2 \,\vert\, w_t(\lambda),y,\bm{\xi}^{(1)}]}_{(2)}.
\end{aligned}    
\end{equation}

\paragraph{Bound for term (1) of \eqref{eq:errordecom_i}} We have that
\begin{align*}
    \EE[\norm{v_k}^2 \,\vert\, w_t(\lambda),y, \bm{\xi}^{(1)}] &\leq 2\EE[\norm{v_k - \bar v}^2 + \norm{\bar v}^2\,\vert\, w_t(\lambda),y, \bm{\xi}^{(1)}] \\[0.8ex]
    &\leq 2\big(\EE[\norm{v_k - \bar v}^2\,\vert\, w_t(\lambda),y, \bm{\xi}^{(1)}] + \norm{y}^2/(1-q)^2 \big) \\[0.8ex]
    &\leq 2 \norm{y}^2 \big(\sigma_\lambda(k) + 1/(1-q)^2 \big).
\end{align*}
where in the second last inequality we used \Cref{ass:new}\ref{ass:new_ii}. Hence 
\begin{equation*}
(1) \leq 4 \norm{y}^2 \big(\sigma_\lambda(k) + 1/(1-q)^2 \big) 
\norm{\partial_2 \bar{\Phi}(\hz) - A_2}^2.
\end{equation*}
Now recall that
\begin{align*}
    \partial_2 \bar{\Phi}(\hz) &= \partial_1 G(\bar{T}(\hz),\lambda) \partial_2 \bar{T}(\hz) +  \partial_2 G(\bar{T}(\hz),\lambda), \quad
    A_2 = B_1C_2 + B_2,
\end{align*}
therefore we have
\begin{align*}
\norm{\partial_2 \bar{\Phi}(\hz)- A_2}
&\leq \norm{\partial_1 G(\bar{T}(\hz),\lambda) \partial_2 \bar{T}(\hz) - \partial_1 G(\bar{T}(\hz),\lambda)C_2 }\\[0.8ex]
&\qquad + \norm{\partial_1 G(\bar{T}(\hz),\lambda)C_2 - B_1C_2}
+\norm{\partial_2 G(\bar{T}(\hz),\lambda)- B_2}\\[0.8ex]
&\leq \norm{\partial_1 G(\bar{T}(\hz),\lambda)}\norm{\partial_2 \bar{T}(\hz)- C_2}\\[0.8ex]
&\qquad+ \norm{C_2}\norm{\partial_1 G(\bar{T}(\hz),\lambda) - B_1}
+\norm{\partial_2 G(\bar{T}(\hz),\lambda)- B_2}\\[0.8ex]
&\leq \norm{\partial_2 \bar{T}(\hz) - \partial_2 T(\hz)}+ \norm{ \partial_2 T(\hz) - C_2}\\[0.8ex]
&\qquad+ \norm{B_2}\norm{\partial_1 G(\bar{T}(\hz),\lambda)- B_1}
+\norm{\partial_2 G(\bar{T}(\hz),\lambda)- B_2}\\[-0.5ex]
&\overset{(*)}{\leq} \norm{\partial_2 \bar{T}(\hz) - \partial_2 T(\hz)}+ C_{T,\lambda}\norm{w_t(\lambda)-w(\lambda)}\\[0.8ex]
&\qquad+ C_{G,\lambda}(1+\norm{C_2}) (\norm{\bar{T}(\hz) - T(\hz)}+\norm{T(\hz) - T(\sz)})\\[0.8ex]
&\leq \norm{\partial_2 \bar{T}(\hz) - \partial_2 T(\hz)}+ [C_{T,\lambda}+C_{G,\lambda}(1+\norm{\D{T, 2}(\sz)}_{\sup})]\Delta_t\\[0.8ex]
&\qquad+ C_{G,\lambda}(1+\norm{\D{T, 2}(\sz)}_{\sup}) \norm{\bar{T}(\hz) - T(\hz)},
\end{align*}
where in ${\mathrm{(*)}}$ we used \eqref{eq:boundab} and in the last inequality the the fact that  $C_2 \in \D{T, 2}(\sz)$.
Hence, from Assumption~\ref{ass:variance} and \eqref{eq:varhatt}, we have
\begin{align*}
\EE\big[\norm{\partial_2 \bar{\Phi}(\hz)- C_2\big}^2\,\big\vert\, w_t(\lambda),y\big]
&\leq 3 \Var[\partial_2 \bar{T}(\hz)\,\vert\, w_t(\lambda),y]+ 3 [C_{T,\lambda}+C_{G,\lambda}(1+\norm{\D{T, 2}(\sz)}_{\sup})]^2\Delta_t^2\\[0.8ex]
&\qquad + 3C^2_{G,\lambda}(1+\norm{\D{T, 2}(\sz)}_{\sup})^2 \Var[\bar{T}(\hz)\,\vert\, w_t(\lambda),y]\\[0.8ex]
&\leq \frac{3\sigma_2^\prime}{J} +3 [C_{T,\lambda}+C_{G,\lambda}(1+\norm{\D{T, 2}(\sz)}_{\sup})]^2\Delta_t^2\\[0.8ex]
&\qquad +  3C^2_{G,\lambda}(1+\norm{\D{T, 2}(\sz)}_{\sup})^2 \frac{b_{\lambda}(\Delta_t^2)}{J}.
\end{align*}
In the end we have
\begin{equation*}
\EE[(1) \,\vert\, w_t(\lambda), y] \leq 12 \norm{y}^2\bigg(\sigma_\lambda(k) + \frac{1}{(1-q)^2} \bigg) 
\bigg(\frac{\sigma_2^\prime}{J} + [C_{T,\lambda}+C_{G,\lambda}M_{T,\lambda}]^2\Delta_t^2
+ C^2_{G,\lambda}M_{T,\lambda}^2 \frac{b_{\lambda}(\Delta_t^2)}{J}
 \bigg),
\end{equation*}
where we set $M_{T,\lambda} = 1+ \norm{\D{T,2}(w(\lambda),\lambda)}_{\sup}$.

\paragraph{Bound for term (2) of \eqref{eq:errordecom_i}} We have
\begin{equation*}
\norm{v_k- (I- A_1^\top)^{-1} y} \leq \norm{v_k- \bar{v}} + \norm{\bar v- (I- A_1)^{-\top} y}.
\end{equation*}
Let $\hat{B}_1 = \partial_1 G(\bar{T}(\hz),\lambda)$ and $\hat{C}_1 = \partial_1 T(\hz)$ and $\hat{A}_1 = \hat{B}_1 \hat{C}_1$ and recall that $\bar v = (I- \hat{A}_1^\top)^{-1}y$ and $A_1 \in \D{\Phi,1}(\sz)$.
Noting that $\max\{\norm{B_1}, \norm{C_1}, \norm{\hat{B}_1}, \norm{\hat{C}_1})\} \leq 1$, $\max\{\norm{\hat{A}_1}, \norm{A_1}\} \leq q$ and hence $\max\{\norm{(I- \hat{A}_1^\top)^{-1}}, \norm{(I-A_1^\top)^{-1}}\} \leq 1/(1-q)$, we obtain
\begin{align*}
\norm{v_k&- (I - A_1^\top)^{-1} y}\\[0.8ex] 
&\leq \norm{v_k- \bar{v}} + \norm{y} \norm{(I - \hat{A}_1^\top)^{-1}}\norm{(I - A_1^\top)^{-1}} \norm{\hat{A}_1 - A_1}\\
 & \leq \norm{v_k- \bar v}+ \frac{\norm{y}}{(1-q)^2} \norm{\partial_1 G(\bar{T}(\hz),\lambda) \partial_1 T(\hz) -
 B_1 C_1}\\
 & \leq \norm{v_k- \bar v}+ \frac{\norm{y}}{(1-q)^2}\Big[ \norm{\partial_1 G(\bar{T}(\hz),\lambda) \partial_1 T(\hz) - 
 B_1 \partial_1 T(\hz)} 
+ \norm{B_1 \partial_1 T(\hz)- 
 B_1 C_1\big)} \Big] \\
  & \leq \norm{v_k- \bar v}+ \frac{\norm{y}}{(1-q)^2}\Big[ \norm{\partial_1 T(\hz)}\norm{\partial_1 G(\bar{T}(\hz),\lambda) - 
 B_1} + \norm{B_1} \norm{  \partial_1 T(\hz)- 
C_1}\Big]\\
  & \leq \norm{v_k- \bar v}+ \frac{\norm{y}}{(1-q)^2}\Big[ \norm{\partial_1 G(\bar{T}(\hz),\lambda) - 
 B_1}+ \norm{\partial_1 T(\hz) - 
C_1}\Big]\\
  & \overset{(*)}{\leq} \norm{v_k- \bar v}+ \frac{\norm{y}}{(1-q)^2}\big[ C_{G,\lambda} (\norm{\bar{T}(\hz)-T(\hz)} + \norm{T(\hz)-T(\sz)})
+ C_{T,\lambda} \norm{w_t(\lambda)- w(\lambda)} \big]\\
  & \leq \norm{v_k- \bar v}+ \frac{\norm{y}}{(1-q)^2}\big[ C_{G,\lambda} \norm{\bar{T}(\hz)-T(\hz)} 
+ (C_{G,\lambda}+C_{T,\lambda}) \Delta_t\big],
\end{align*}
where in ${\mathrm{(*)}}$ we used  \eqref{eq:boundab}. Therefore,
\begin{align*}
\EE\big[\!\gap(v_k, &(I- A_1^\top)^{-1} y)^2 \,\big\vert\, w_t(\lambda),y, \bm{\xi}^{(1)}\big] \\
&\leq 3 \bigg( \norm{y}^2\sigma_\lambda(k)+ \frac{\norm{y}^2}{(1-q)^4}\big[ C^2_{G,\lambda} \norm{\bar{T}(\hz)-T(\hz)}^2 
+ (C_{G,\lambda}+C_{T,\lambda})^2 \Delta_t^2\big] \bigg)
\end{align*}
and hence, taking the expectation over $\bm{\xi}^{(1)}$ we obtain
\begin{align*}
\EE\big[\!\gap(v_k, &(I- A_1^\top)^{-1} y)^2 \,\big\vert\, w_t(\lambda),y\big]\\[0.8ex]
&\leq 3 \norm{y}^2 \bigg( \sigma_\lambda(k)+ \frac{1}{(1-q)^4}\big[ C^2_{G,\lambda} \Var[\bar{T}(\hz)\,\vert\, w_t(\lambda),y]
+ (C_{G,\lambda}+C_{T,\lambda})^2 \Delta_t^2\big] \bigg)\\
&  = 3 \norm{y}^2 \bigg( \sigma_\lambda(k)+ \frac{1}{(1-q)^4}\Big( C^2_{G,\lambda} \frac{\Var[\hat T_\xi(\hz)\,\vert\, w_t(\lambda),y]}{J}
+ (C_{G,\lambda}+C_{T,\lambda})^2 \Delta_t^2\Big) \bigg)\\
&  \leq 3 \norm{y}^2\bigg( \sigma_\lambda(k)+ \frac{1}{(1-q)^4}\Big( C^2_{G,\lambda} \frac{b_{\lambda}(\Delta_t^2)}{J}
+ (C_{G,\lambda}+C_{T,\lambda})^2 \Delta_t^2\Big) \bigg).
\end{align*}
In the end we have
\begin{equation*}
\EE[(2)\,\vert\, w_t(\lambda),y] \leq 6\norm{ \D{\Phi,2}(w(\lambda), \lambda)}^2_{\sup}\norm{y}^2\bigg( \sigma_\lambda(k)+ \frac{1}{(1-q)^4}\Big( C^2_{G,\lambda} \frac{b_{\lambda}(\Delta_t^2)}{J}
+ (C_{G,\lambda}+C_{T,\lambda})^2 \Delta_t^2\Big) \bigg).
\end{equation*}

\paragraph{Combined bound} By combining the above results we finally obtain 
\begin{align*}
&\EE [\gap(\partial_2 \bar{\Phi}(\hz)^\top  v_k, \D{w}^{\imp}(\lambda)^\top y)^2]\,\vert\, w_t(\lambda),y]\\
& \leq 12 \norm{y}^2\big(\sigma_\lambda(k) + \kappa^2  \big) 
\bigg(\frac{\sigma_2^\prime}{J} + [C_{T,\lambda}+C_{G,\lambda}M_{T,\lambda}]^2\Delta_t^2
+ C^2_{G,\lambda}M_{T,\lambda}^2 \frac{\sigma_1 +3\sigma_2(2\Delta_t^2 + a_\lambda^2)}{J}
 \bigg)\\
 &\qquad + 6\norm{y}^2\norm{\D{\Phi,2}(w(\lambda),\lambda)}^2_{\sup}\bigg( \sigma_\lambda(k)+ \kappa^{4}\Big( C^2_{G,\lambda} \frac{\sigma_1 +3\sigma_2(2\Delta_t^2 + a_\lambda^2)}{J}
+ (C_{G,\lambda}+C_{T,\lambda})^2 \Delta_t^2\Big) \bigg),
\end{align*}
where we used the expression for $b_{\lambda}(\Delta_t^2)$ and $\kappa = (1-q)^{-1}$. 
Taking the expectation $\EE[\cdot\,\vert\, w_t(\lambda)]$ and recalling the hypothesis on $\norm{y}^2$ and $\Delta^2_t$ in \Cref{ass:new}\ref{ass:new_i}\ref{ass:new_iii}, the statement follows.
\end{proof}

Before reporting the proof for the linear system rate, we rewrite for reader's convenience the following result from \citep{grazzi2021convergence}, which establishes a convergence rate for stochastic fixed-point iterations with a decreasing step size.

\begin{lemma}{\citep[Theorem 4.2]{grazzi2021convergence}}
\label{lm:20240310}
Let $\Psi\colon \R^d \to \R^d$ be a $q$-contraction ($0\leq q<1$), $\xi$ a random variable with values in $\Xi$ and $\hat{\Psi}\colon \R^d\times \Xi \to \R^d$
be such that for every $v \in \R^d$
\begin{equation*}
    \EE[\hat{\Psi}(v,\xi)]= \Psi(v)\quad\text{and}\quad
    \Var[\hat{\Psi}(v,\xi)] \leq \hat{\sigma}_1+ \hat{\sigma}_2\norm{\Psi(v)-v}^2,
\end{equation*}
for some $\hat{\sigma}_1,\hat{\sigma}_2>0$.
Let $\eta_i = \beta/(\gamma+i)$, with $\beta>1/(1-q^2)$ and $\gamma\geq \beta(1+\hat \sigma_2)$. Let $(\xi_i)_{i \in \N}$ be a sequence of i.i.d copies of $\xi$ and let $(v_i)_{i \in \N}$ be such that $v_0=0$ and for $i=0,1,\dots$
\begin{equation*}
    v_{i+1} = v_i + \eta_i ( \hat{\Psi}(v_i, \xi_i) - v_i).
\end{equation*}
Then for every $i \in \N$
\begin{equation*}
    \EE[\norm{v_i - \bar{v}}^2] \leq \frac{1}{\gamma+i} \max\Big\{ \gamma \norm{\bar{v}}^2, \frac{\beta^2 \hat\sigma_1}{\beta(1-q^2)-1} \Big\},
\end{equation*}
where $\bar{v}$ is the (unique) fixed point of $\Psi$.
\end{lemma}

We now present the rate for the algorithm used to solve the linear system in \Cref{alg:nsid}. Consider the procedure in \Cref{alg:lssolver}
\begin{algorithm}[ht]
   \caption{Stochastic fixed point iterations}
   \label{alg:lssolver}
\begin{algorithmic}[1]
   \STATE {\bfseries Input:} $k\in \N$,  $u_1, u_2, y \in \R^d$, $\bm{\xi}^{(2)} = (\xi^{(2)}_i)_{1 \leq i \leq k}$.
   \STATE $\hat \Psi\colon (v,x) \mapsto  \partial_1 \hat T(u_1,\lambda, x)^\top \partial_1 G(u_2, \lambda)^\top v + y $
    \STATE $v_0 = 0$
    \FOR{$i=1$ {\bfseries to} $k$}
    \STATE $v_i \gets (1- \eta_i)v_{i-1} + \eta_i \hat \Psi(v_{i-1}, \xi^{(2)}_{i})$
    \ENDFOR
    \STATE {\bfseries Return}  $v_k$
\end{algorithmic}
\end{algorithm}

Note that $v_k$ in \Cref{alg:nsid} is exactly the output of \Cref{alg:lssolver} with $u_1 = w_t(\lambda)$, $u_2 = \bar{T}_t(\lambda)$. Moreover, we obtain the following convergence rate which is completely independent from the inputs $u_1$ and $u_2$.

\begin{lemma}[Linear system rate]\label{lm:linsysrate} Under \Cref{ass:contlip} and \ref{ass:variance},
let $\hat \sigma_2 = 2\sigma'_1(1-q)^{-2}$, $\hat \sigma_1 = \hat \sigma_2\norm{y}^2$,  
and consider the stochastic fixed point iterations in \Cref{alg:lssolver} 
with $\eta_i=\beta /(\gamma+i)$, with $\beta>1 /\left(1-q^2\right)$ and $\gamma \geq \beta\left(1+\hat\sigma_2\right)$. For any $u_2, u_1, y \in \R^d$ let the solution of the linear system be
\begin{align*}
    \bar v := (I - \partial_1 T(u_1,\lambda)^\top \partial_1 G(u_2, \lambda)^\top )^{-1} y.
\end{align*}
Then we have
\begin{equation}
\label{eq:20240311a}
    \EE[\norm{v_k - \bar v}^2] \leq  \frac{\norm{y}^2}{\gamma + k} \max \left\{ \frac{\gamma}{(1-q)^2}, \frac{\beta^2 \hat \sigma_2}{\beta\left(1-q^2\right)-1}\right\}.
\end{equation}
In particular, if we set $\beta=2/(1-q^2), \gamma=2(1+\hat \sigma_2)/(1-q^2)$, we obtain
\begin{equation*}
    \EE[\norm{v_k - \bar v}^2] \leq \frac{1}{k}\cdot \frac{2\norm{y}^2(1+ 4\sigma_1')}{(1-q)^5}. 
\end{equation*}
\end{lemma}
\begin{proof}
Let 
\begin{equation*}
\Psi(v) := \EE [\hat \Psi(v, \xi)] = \partial_1 T(u_1,\lambda)^\top \partial_1 G(u_2, \lambda)^\top  v + y.     
\end{equation*}
Since $\norm{\partial_1 T(u_1,\lambda)^\top \partial_1 G(u_2, \lambda)^\top} \leq q$,
 $\Psi$ is a $q$-contraction with fixed point $\bar v$. 
 It is immediate to see that
 \begin{equation*}
    \Var[\hat \Psi(v, \xi)] \leq
    \Var[\partial_1 \hat T_\xi(u_1,\lambda)] \norm{v}^2.
 \end{equation*}
 Moreover, we have
 \begin{equation*}
     \norm{v} \leq \norm{v-\Psi(v)} + \norm{\Psi(v)-\Psi(0)}+ \norm{\Psi(0)}
     \leq 
     \norm{v-\Psi(v)}+ q\norm{v} + \norm{y}
 \end{equation*}
 and hence 
 \begin{equation}
\label{eq:20240310a}     
(1-q)\norm{v} \leq \norm{v-\Psi(v)} + \norm{y}, 
 \end{equation}
 which, recalling Assumption~\ref{ass:variance} on the variance of $T_1'$, ultimately yields
\begin{equation*}
    \Var[\hat \Psi(v, \xi)] \leq \frac{2}{(1-q)^2}\Var[\partial_1 \hat T_\xi(u_1,\lambda)]\big(\norm{v- \Psi(v)}^2 + \norm{y}^2 \big)
    \leq \frac{2\sigma_1^\prime}{(1-q)^2}\norm{\Psi(v)-v}^2 + \frac{2\sigma_1^\prime \norm{y}^2}{(1-q)^2}.
\end{equation*}
Therefore, the first part of the statement follows from \Cref{lm:20240310}  and from $\norm{\bar{v}} \leq \norm{y}(1-q)^{-1}$ (a consequence of \eqref{eq:20240310a}). The last part follows 
by \eqref{eq:20240311a}, the equations
\begin{equation*}
    \gamma = \frac{2}{1-q^2}\Big(1 + \frac{2\sigma_1^\prime}{(1-q)^2}\Big) \leq \frac{2(1+2\sigma_1^\prime)}{(1-q^2)(1-q)^2} \quad\text{and}\quad \beta^2\hat{\sigma}_2= \frac{8\sigma_1^\prime}{(1-q)^2(1-q^2)^2}
\end{equation*}
and the fact that $(1-q^2)^{-1} \leq (1-q)^{-1}$ when $q<1$.
\end{proof}

\begin{proof}[\proofname{} of \Cref{cor:nsidrate}]
By applying \Cref{lm:linsysrate} with $u_1 = w_t(\lambda)$ and $u_2 = \bar{T}(w_t(\lambda),\lambda)$ we obtain that \Cref{ass:new}\ref{ass:new_ii} (the rate on the mean square error of $v_k$) is satisfied with $\sigma_\lambda(k) = O(\kappa^5 k^{-1})$. The statement follows by applying \Cref{th:nsidrate} and substituting the rates $\rho_\lambda(t)$ and $\sigma_\lambda(k)$.    
\end{proof}

\section{Bilevel Optimization}
\label{app:BO}

In this section we consider Problem~\eqref{eq:bilevel} and we make the following assumption.
\begin{assumption}\label{ass:elipgradsel}
The map $E$ satisfies Assumption~\ref{ass:lipsel}\ref{ass:lipsel_i} with constant $L_E$ and corresponding conservative derivative  $\D{E}$. 
\end{assumption}

Note that similarly to $\Phi$, since $E$ satisfies Assumption~\ref{ass:elipgradsel}, a direct application of \Cref{lm:glolip} to the map $E$ yields 
\begin{lemma}\label{lm:lipE}
Under \Cref{ass:elipgradsel}, for every $\lambda \in \Lambda$, there exist $R_{E,\lambda} > 0$ such that for every $u \in \R^d$ 
\begin{equation*}
\gap(\D{E}(u,\lambda),\D{E}(w(\lambda),\lambda)) \leq \LE\norm{u- w(\lambda)},
\end{equation*}
where $\LE := L_E +M_{E,\lambda}/R_{E,\lambda}$,
with $M_{E,\lambda} := \max_{i\in \{1,\dots,m\}}\min_{j \in I_E(w(\lambda),\lambda)} \norm{E'_i(w(\lambda),\lambda)-E'_j(w(\lambda),\lambda)}$.
\end{lemma}

\subsection{Deterministic Case}

\begin{theorem}\label{th:aiditdhgrates}
Let \Cref{ass:lipsel} and \ref{ass:elipgradsel} hold. Then for every $\lambda \in \Lambda$ and every $t,k \in \N$ we have that
for BAID-FP we get
\begin{equation*}
     \gap(\D{f_t}^k(\lambda), \D{f}^{\fix}(\lambda)) = O(\kappa e^{-k/\kappa} +\kappa^2 \Delta_t)
\end{equation*}
while if $w_t(\lambda) = \Phi(w_{t-1}(\lambda),\lambda)$, for BITD we get
\begin{equation*}
     \gap(\D{f_t}(\lambda), \D{f}^{\fix}(\lambda)) = O(\kappa t e^{-\kappa/t}).
\end{equation*}
\end{theorem}
\begin{proof}
For simplicity, let $\A = \D{E}(w(\lambda),\lambda)$, $\A_t = \D{E}(w_t(\lambda),\lambda)$ and recall that
\begin{align*}
    \D{f_t}(\lambda) = \A_t (\D{w_t}(\lambda)), \qquad \D{f}^{\fix}(\lambda) = \A (\D{w}^{\fix}(\lambda)).
\end{align*}
Using the properties of excess in \Cref{lm:gapprop} we obtain, for BITD:
\begin{align*}
    \gap(\D{f_t}(\lambda), \D{f}^{\fix}(\lambda) ) &\leq \gap(\A_t (\D{w_t}(\lambda)), \A_t (\D{w}^{\fix}(\lambda)) ) 
    + \gap(\A_t (\D{w}^{\fix}(\lambda)), \A (\D{w}^{\fix}(\lambda)) )  \\[0.8ex]
    &\leq
    \norm{\D{E,1}(w_t(\lambda),\lambda)}_{\sup}\gap(\D{w_t}(\lambda),\D{w}^{\fix}(\lambda))  \\[0.8ex]
    &\quad+\left(1+\norm{\D{w}^{\fix}(\lambda)}_{\sup}\right) \gap(\D{E}(w_t(\lambda),\lambda), \D{E}(w(\lambda),\lambda)) \\[0.8ex]
    &\leq (\norm{\D{E,1}(w(\lambda),\lambda)}_{\sup} + \LE\Delta_t) \gap(\D{w_t}(\lambda),\D{w}^{\fix}(\lambda)) \\[0.8ex]
    &+ \Big(\frac{\norm{\D{\Phi,2}(w(\lambda),\lambda)}_{\sup}}{1-q} + 1\Big)\LE\Delta_t \\ %
    &\leq (\norm{\D{E,1}(w(\lambda),\lambda)}_{\sup} + \LE \Delta_0) \times \underbrace{O(\kappa t e^{-t/\kappa})}_{(*)} \\
    &+ \Big(\frac{\norm{\D{\Phi,2}(w(\lambda),\lambda)}_{\sup}}{1-q} + 1\Big)\LE\Delta_0 e^{-t/\kappa},
\end{align*}
where we used $\Delta_t \leq \Delta_0 e^{-t/\kappa} < \Delta_0$, the ITD bound in \Cref{th:itdaidrates} and \Cref{lm:lipE}.
A very similar proof can be done for AID-FP by changing the $(*)$ term to $O(\kappa e^{-k/\kappa} +\kappa^2 e^{-t/\kappa}).$
\end{proof}

\subsection{Stochastic Case}
We consider the special case of  Problem~\eqref{eq:bilevel} with 
\begin{equation*}
    E(w,\lambda) = \EE[\hat E_\zeta(w,\lambda)], \ \Phi(w,\lambda) = G(\EE[\hat T_\xi(w,\lambda)], \lambda).
\end{equation*}
In addition to \Cref{ass:elip}, as for the smooth case in \citep{grazzi2023bilevel}, we consider the following assumption on $E$
\begin{assumption}\label{ass:elip}
For any $\lambda \in \Lambda$ there exists $B_{E,\lambda}\geq 0$ such that 
\begin{equation*}
\forall\,w \in \R^d\colon \ \norm{\D{E,1}(w,\lambda)}_{\sup} \leq B_{E,\lambda}.
\end{equation*}
\end{assumption}
The assumption above is verified e.g., for the logistic and for the cross-entropy loss.
Moreover, the assumptions on $\hat E$ are the following.
\begin{assumption}\label{ass:Estoch} $\zeta$ is a random variable with values in $\mathcal{Z}$ and for every $z \in \mathcal{Z}$  
\begin{enumerate}[label={\rm (\roman*)}]
\item $\hat{E}_z\colon \R^d\times O_{\Lambda} \to \R^d$,
$\EE[\hat{E}_\zeta(u,\lambda)] = E(u,\lambda)$.
\item 
$\hat{E}_z$ is path differentiable with conservative derivative $ D_{\hat E}$
and $E_z'$ is a selection of $ D_{\hat E}$
such that $\hat E_z'(u,\lambda) = [\partial_1 \hat E_z(u,\lambda), \partial_2 \hat E_\zeta(u,\lambda)]$ 
and there exist $\sigma_{E,1}, \sigma_{E,2} \geq 0$ such that for every $u \in \R^d$ and $\lambda\in \Lambda$ 
\begin{gather*}
\EE\big[E_\zeta'(u,\lambda)\big] = E'(u,\lambda)\in \D{E}(u,\lambda), \quad
\Var[\partial_1 \hat E_\zeta(u,\lambda)] \leq \sigma_{E,1}, \quad \Var[\partial_2 \hat E_\zeta(u,\lambda)] \leq \sigma_{E,2}.
\end{gather*}
\end{enumerate}
\end{assumption}

\begin{theorem}\label{th:nsidhgrates} Let \Cref{ass:contlip}, \ref{ass:variance}, \ref{ass:elipgradsel}, \ref{ass:elip}, \ref{ass:Estoch} hold
and let $\kappa = (1-q)^{-1}$. Also, suppose that $\EE[\norm{w_t(\lambda) - w(\lambda)}]\leq \rho_\lambda(t)$, for every $t \in \N$. Then the output $\hat{\nabla} f(\lambda)^\top$
 of NSID-Bilevel (\Cref{alg:nsid-bilevel}) where NSID uses step sizes $\eta_i = \Theta(i^{-1})$ satisfies
\begin{equation*}
    \EE[\gap(\hat{\nabla} f(\lambda)^\top, \D{f}^{\imp}(\lambda))^2] = O\left( \frac{\kappa^5}{k}+ \kappa^4 \left(\frac{1}{J_2} + \rho_\lambda(t)\right)+  \frac{\kappa^2}{J_1} \right). 
\end{equation*}
Furthermore, if $\rho_\lambda(t) = O(\kappa^{\alpha} t^{-1})$ ($\alpha > 0$), then
\begin{equation*}
    \EE[\gap(\hat{\nabla} f(\lambda)^\top, \D{f}^{\imp}(\lambda))^2] = O\big(\kappa^2 J_1^{-1} + \kappa^5(k^{-1} + J_2^{-1} + \kappa^{\alpha}t^{-1})\big). 
\end{equation*}
Therefore, by setting e.g., $t = k = J_1 = J_2$ we have
\begin{equation*}
    \EE[\gap(\hat{\nabla} f(\lambda)^\top, \D{f}^{\imp}(\lambda))^2] = O(\kappa^{5 + \alpha} t^{-1})
\end{equation*}
which has the same dependency on $t$ as stochastic gradient descent on strongly convex and Lipschitz smooth objectives \citep{bottou2018optimization}.
\end{theorem}

\begin{proof}
For simplicity, let $z_t = (w_t(\lambda), \lambda)$, $z = (w(\lambda), \lambda)$, $\A = \D{E}(w(\lambda),\lambda)$, $\B_t = \{\bar{E}'(\hz)\}$.
We also recall that
\begin{align*}
\D{f}^{\imp}(\lambda) := \A (\D{w}^{\imp}(\lambda)), \qquad
\hat{\nabla} f(\lambda)^\top = r(\hz)^\top + \partial_2\bar{E}(\hz),
\end{align*}
with $r(\hz)$ which is an estimator of $\D{w}^{\imp}(\lambda)^\top \partial_1 \bar{E}(\hz)$.
Then, using the properties in \Cref{lm:gapprop} and noting that $\B_t(\D{w}^{\imp}(\lambda)) = \partial_1 \bar{E}(\hz) \D{w}^{\imp}(\lambda) + \partial_2 \bar{E}(\hz)$, we have
\begin{align*}
\gap(&\hat{\nabla} f(\lambda)^\top, \D{f}^{\imp}(\lambda))\\[1ex]
&\leq \gap(\hat{\nabla} f(\lambda)^\top, \B_t(\D{w}^{\imp}(\lambda)) + \gap(\B_t(\D{w}^{\imp}(\lambda)), \A (\D{w}^{\imp}(\lambda))) \\[0.8ex]
&\leq \gap(r(\hz), \D{w}^{\imp}(\lambda)^\top\, \partial_1 \bar{E}(\hz))
+ (1 + \norm{\D{w}^{\imp}(\lambda)}_{\sup})\gap(\bar{E}'(\hz), \D{E}(\sz))  \\[0.8ex]
&\leq \gap(r(\hz), \D{w}^{\imp}(\lambda)^\top\, \partial_1 \bar{E}(\hz))
+ (1 + \norm{\D{w}^{\imp}(\lambda)}_{\sup})(\norm{\bar{E}'(\hz) - E'(\hz)} +\gap(E'(\hz), \D{E}(\sz)))  \\[0.8ex]
&\leq \gap(r(\hz), \D{w}^{\imp}(\lambda)^\top\, \partial_1 \bar{E}(\hz))
+ (1 + \norm{\D{w}^{\imp}(\lambda)}_{\sup})(\norm{\bar{E}'(\hz) - E'(\hz)} +\LE\Delta_t )  
\end{align*}
Moreover, let $\tilde{\EE} = \EE[\,\cdot \given w_t(\lambda)]$, we have that
\begin{align*}
\tilde{\EE}[\norm{\bar{E}'(\hz) - E'(\hz)}^2] &= \tilde{\EE}[\norm{\partial_1 \bar{E}(\hz) - \partial_1 E(\hz)}^2] + \tilde{\EE}[\norm{\partial_2\bar{E}(\hz) - \partial_2E(\hz)}^2]  \\
&\leq \frac{\Var[\partial_1 \hat{E}_\zeta(z_t)\given w_t(\lambda)] + \Var[\partial_2 \hat{E}_\zeta (z_t)\given w_t(\lambda)]}{J_1} \leq \frac{\sigma_{E,1} + \sigma_{E,2}}{J_1}.
\end{align*}
Hence
\begin{align}
\label{eq:20240314a}
\nonumber
\tilde{\EE}[&\gap(\hat{\nabla} f(\lambda)^\top, \D{f}^{\imp}(\lambda))^2]\\[1ex]
&\leq 3\left( \tilde{\EE}[\gap(r(\hz), \D{w}^{\imp}(\lambda)^\top\, \partial_1 \bar{E}(\hz))^2 ]
+ (1+ \norm{\D{w}^{\imp}(\lambda)}_{\sup})^2 \Big(C^2_{E,\lambda}\Delta^2_t + \frac{\sigma_{E,1} + \sigma_{E,2}}{J_1}\Big) \right).
\end{align}
We also note that $\norm{\D{w}^{\imp}(\lambda)}_{\sup} \leq \norm{\D{\Phi,2}(w(\lambda),\lambda)}_{\sup}/(1-q)$ and that
\begin{equation}\label{eq:ebaruseful}
    \tilde{\EE} \big[\norm{\partial_1\bar{E}(\hz)}^2\big] \leq 
        2\tilde{\EE}\big[\norm{\partial_1\bar{E}(\hz) - \partial_1 E(\hz)}^2\big] +
    2\norm{\D{E,1}(\hz)}^2_{\sup} \leq \frac{2\sigma_{E,1}}{J_1}+ 2 B_{E,\lambda}.
\end{equation}
Therefore, taking the total expectation in \eqref{eq:20240314a} and applying Theorem~\ref{th:nsidrate} with $y = \partial_1\bar{E}(\hz)$ we get
\begin{align*}
\EE[\gap(\hat{\nabla} &f(\lambda)^\top, \D{f}^{\imp}(\lambda))^2]\\[1ex]
&\leq O\left( \sigma_\lambda(k)+ \kappa^4 \bigg(\frac{2\sigma_{E,1}}{J_1}+ 2 B_{E,\lambda} \bigg)\left(\frac{1}{J_2} + \rho_\lambda(t)\right)\right) \\[0.8ex]
&\qquad +  3 (1+ \kappa \norm{\D{\Phi,2}(w(\lambda),\lambda)}_{\sup})^2 \Big(C^2_{E,\lambda}\Delta^2_t + \frac{\sigma_{E,1} + \sigma_{E,2}}{J_1}\Big) \big)\\[0.8ex]
& = O\left( \sigma_\lambda(k)+ \kappa^4 \left(\frac{1}{J_2} + \rho_\lambda(t)\right)+  \frac{\kappa^2}{J_1} \right).
\end{align*}
The first part of the statement follows by noting that for NSID we have $\sigma_\lambda(k)  = O(\kappa^5/k)$, where $\kappa = 1/(1-q)$.
The second and last result are immediate.
\end{proof}

\section{Experimental Details}
\label{app:exps}
We give more information on the numerical experiments in \Cref{se:exp}.

\subsection{Computing the approximation Error.}\label{se:computingapproxerror} 
Let $c \in \R^{m}$, be the output of an algorithm approximating the jacobian vector product $\D{w}^{\fix}(\lambda)^\top y$. We call approximation error the quantity
\begin{equation*}
    \gap(c, \D{w}^{\fix}(\lambda)^\top y).
\end{equation*}
Since $\D{w}^{\fix}(\lambda)^\top y$ is set valued and each element is not available in closed form, we instead approximate an upper bound to this quantity using AID-FP for enough iterations $k$, which as we mention in \Cref{se:detitdaid}, generates a subsequence linearly converging to an element of $\D{w}^{\fix}(\lambda)^\top y$. Also, as a starting point to AID-FP we use $w_t(\lambda) = \Phi(w_{t-1}(\lambda),\lambda)$, with sufficiently large $t$ and starting from $w_0(\lambda) = 0$, so to be sufficiently close to the fixed point solution $w(\lambda)$, also not available in closed form.

\subsection{Constructing the fixed-point map.}
In all the experiments, we consider composite minimization problems in the form 
\begin{equation*}
    \min_{u} f_\lambda(u) + g_\lambda(u).
\end{equation*}
To convert it to fixed point we set a step size $\eta_\lambda > 0$ and set
\begin{equation*}
    \Phi(u, \lambda) = G(T(u,\lambda), \lambda), 
\end{equation*}
with
\begin{equation*}
    G(u,\lambda) = \prox_{\eta_\lambda g_\lambda}(u) \qquad 
    T(w,\lambda) = u -\eta_\lambda \nabla f_\lambda(u)
\end{equation*}
In particular, since in our case $g_\lambda$ is always the an elastic net regularization, $\prox$ is the soft-thresholding and we set 
\begin{equation*}
    \eta_\lambda = \frac{2}{c (L+ \mu) + 2\lambda_2},
\end{equation*} where $\lambda_2$ is the L2 regularization parameter and $L, \mu$ are the largest and smallest eigenvalues of $n^{-1} X^\top X$, where $X$ is the design matrix of the training set (which contains the corrupted points in the case of data poisoning). We set $c=1$ for the elastic net experiments with synthetic data, while we set $c=0.1$ for the poisoning experiments. The choice $c=1$ yields the optimal theoretical value for the step-size when using the square loss, while we used a difference choice for data poisoning (using the cross-entropy loss) since we found the optimal theoretical value for the step-size too conservative. To set the stochastic step-size schedules we also set the (estimated) contraction constant as $q = \max\{\lvert 1- \eta_{\lambda} (c L +\lambda_2) \rvert, \lvert 1- \eta_{\lambda} (c \mu +\lambda_2) \rvert\}$.

\subsection{Details for the AID and ITD Experiments}
We construct the synthetic dataset by sampling each element of the matrix $X \in \R^{n \times d}$ and the vector $w$ from a normal distribution. Subsequently, we set the non-informative features of $w$ to zero and we compute the vector $y$ as $y = Xw + \epsilon$, where $\epsilon_i$ is Gaussian noise with mean $0.1$ and unit variance. For this experiment we set $n = 100$ and $p=100$ of which $30$ are informative. We use $200$ hold-out examples for the validation set.

\subsection{Details for the Stochastic Experiments}
We start by noting that for each point in the plots for the stochastic experiments in \Cref{fig:hdstoch} $w_t(\lambda)$ is fixed as the last iterate of deterministic iterative soft-thresholding, so that the focus lies entirely on the computation of the derivative.

For elastic net, we enhance the setup used for the deterministic methods by sampling the population covariance matrix randomly for the informative features. To do so, we first sample a matrix $A_1$ from a standard normal, then we normalize all eigenvalues by diving all of them by their maximum obtaining $A_2$, finally we use the normalized $A_2^\top A_2$ as the covariance matrix of a Gaussian distribution for the informative features. This introduces correlations among the features, thereby increasing the complexity of the problem. We also increase the number of training points from $100$ to $10K$ and the number of validation points from $200$ to $20K$.

For the data poisoning setup we use the MNIST dataset. We split the MNIST original train set into $30K$ example for training and $30K$ examples for validation. 
Additionally, we perform a random split of the training set into $X \in \R^{n \times d}$ and $\tilde{X} \in \R^{n'\times p}$, with $p=784$ representing the number of features for MNIST images. Notably, $n' = 9K$ denotes the number of corrupted examples. It is essential to highlight that $\Gamma \in \R^{n' \times p}$ and $n'p$ is approximately $7$ million, posing a significant challenge for derivative estimation using zero-order methods.
We set the regularization parameters $\lambda = (0.02, 0.1)$ since with this setup, the final uncorrputed linear model achieves a validation accuracy of around $80\%$ with around $90\%$ of components set to zero.

We note that NSID and SID require to choose the step sizes $(\eta_i)$, which we found to be difficult, since the theoretical values are often conservative estimates for this problem. We try two policies: constant and decreasing (as $\Theta(1/i)$) step sizes, indicated with ``const'' and ``dec'' after the method name respectively. Note that only when the step sizes are decreasing NSID is guaranteed to converge. To simplify the setup, we always have the same $\eta_0$ for both constant and decreasing step-size policies. Moreover, we set the step size of SID equal to that of NSID, when they use the same step sizes policy. 
More speifically, we set $\eta_i = a_1/(a_2 + i)$ for (N)SID dec and $\eta = a_1/a_2$ for (N)SID const, where $a_1 = b_1 \beta$ and $a_2 = b_2 \beta$, where beta is set to the theoretical value suggested in \Cref{lm:linsysrate} ($2/(1-q^2)$). We tuned $a_1, a_2$ manually for each setting. In particular we set $a_1=0.5, a_2=2$ for the synthetic Elastic net experiment and $a_1=2, a_2=0.01$ for Data poisoning.

\end{document}